\documentclass[11pt]{article}

\newif\iflong\longfalse
\newif\ifcolt\coltfalse

\usepackage{amsmath}
\usepackage{amssymb}
\usepackage{mathtools}
\usepackage{enumerate}
\usepackage{pgfplots}
\usepackage{pgfplotstable}
\usepgfplotslibrary{groupplots}

\usepackage{dsfont}
\usepackage{enumitem}
\usepackage{wrapfig}
\usepackage{booktabs}
\usepackage{float}
\usepackage{natbib}
\usepackage[boxed]{algorithm}
\usepackage[algo2e,ruled]{algorithm2e}
\usepackage[bf,font=small,skip=1pt]{caption}
\usepackage{url}
\ifcolt
\else
\usepackage{fullpage}
\usepackage{amsthm}
\usepackage{times}
\usepackage[colorlinks,citecolor=blue,urlcolor=blue]{hyperref}
\fi
\usepackage[capitalize]{cleveref}

\newcommand{\E}{\mathbb E}

\newcommand{\sr}[1]{\stackrel{#1}}
\ifcolt
\renewcommand{\set}[1]{\left\{#1\right\}}
\else
\newcommand{\set}[1]{\left\{#1\right\}}
\fi
\newcommand{\ind}[1]{\mathds{1}\!\!\set{#1}}
\newcommand{\argmax}{\operatornamewithlimits{arg\,max}}
\newcommand{\argmin}{\operatornamewithlimits{arg\,min}}

\newcommand{\ceil}[1]{\left \lceil {#1} \right\rceil}
\newcommand{\eqn}[1]{\begin{align}#1\end{align}}

\newcommand{\eq}[1]{\begin{align*}#1\end{align*}}
\newcommand{\logp}{\log_{+}\!}

\def\subsubsect#1{\vspace{1ex plus 0.5ex minus 0.5ex}\noindent{\normalsize\textbf{#1.}}}

\renewcommand{\P}[1]{\mathbb{P}\left\{#1\right\}}
\newcommand{\Pp}[1]{\mathbb{P}'\left\{#1\right\}}
\newcommand{\Pn}[2]{\mathbb{P}_{#1}\left\{#2\right\}}
\newcommand{\Ps}{\mathbb{P}}

\newcommand{\KL}{\operatorname{KL}}

\newcommand{\R}{\mathbb R}

\newcommand{\bonus}[2]{\sqrt{\frac{\alpha}{T_{#1}(#2)} \log\left(\frac{\psi n}{#2}\right)}}
\newcommand{\conf}[4]{\sqrt{\frac{2\gamma#4}{T_{#1}(#2)} \log\left(\frac{1}{#3}\right)}}

\let\epsilon\varepsilon

\ifcolt

\fi

\ifcolt
\else
\theoremstyle{plain}
\newtheorem{theorem}{Theorem}

\newtheorem{lemma}[theorem]{Lemma}

\theoremstyle{definition}

\newtheorem{remark}[theorem]{Remark}

\theoremstyle{remark}
\fi
\newcommand{\CONref}[1]{(\ref{#1})}
\newcommand{\CONreft}[2]{(\ref{#1},\ref{#2})}

\ifcolt
\else

\fi

\ifcolt
\title{
Optimally Confident UCB: Improved Regret for Finite-Armed Bandits
}
 \coltauthor{\Name{Tor Lattimore} \Email{abc@sample.com}\\
 \addr Address 1
 }
\else
\title{
\Large \bf Optimally Confident UCB: Improved Regret for Finite-Armed Bandits
}
\author{{Tor Lattimore} \\[3mm]
\normalsize University of Alberta, Canada \\
\normalsize\texttt{tor.lattimore@gmail.com}
}
\fi

\begin{document}

\pgfplotstableread[comment chars={\%}]{data/exp1.txt}{\dataOne}
\pgfplotstableread[comment chars={\%}]{data/exp2.txt}{\dataTwo}
\pgfplotstableread[comment chars={\%}]{data/exp4.txt}{\dataFour}
\pgfplotstableread[comment chars={\%}]{data/exp5.txt}{\dataFive}
\pgfplotstableread[comment chars={\%}]{data/exp6.txt}{\dataSix}
\pgfplotstableread[comment chars={\%}]{data/exp7.txt}{\dataSeven}
\pgfplotstableread[comment chars={\%}]{data/exp8.txt}{\dataEight}
\pgfplotstableread[comment chars={\%}]{data/exp9.txt}{\dataNine}
\pgfplotstableread[comment chars={\%}]{data/exp10.txt}{\dataTen}
\pgfplotstableread[comment chars={\%}]{data/exp11.txt}{\dataEleven}
\pgfplotstableread[comment chars={\%}]{data/exp12.txt}{\dataTwelve}

\definecolor{C1}{RGB}{0,0,0}
\definecolor{C2}{RGB}{228,26,28}
\definecolor{C3}{RGB}{55,126,184}
\definecolor{C4}{RGB}{77,175,74}
\definecolor{C5}{RGB}{152,78,163}
\definecolor{C6}{RGB}{255,255,51}
\definecolor{C7}{RGB}{255,127,0}

\pgfplotsset{cycle list={
    {red,dotted}, 
    {green!50!black,dashdotdotted}, 
    {blue,dashed}, 
    {black}, 
    {purple}, 
    {orange}, 
    {green!50!black}}}
\pgfplotsset{cycle list={
    {C1,dashed},
    {C2,solid}, 
    {C3,loosely dashdotdotted}, 
    {C4,dotted}, 
    {C5,densely dashdotted}, 
    {brown,densely dashdotted}, 
    {brown,densely dashdotted}, 
    }}

\pgfplotsset{every axis plot/.append style={line width=1.4pt}}

\date{}

\maketitle

\begin{abstract}
I present the first algorithm for stochastic finite-armed bandits that simultaneously enjoys 
order-optimal problem-dependent regret and worst-case regret. Besides the theoretical results, the new algorithm
is simple, efficient and empirically superb.
The approach is based on UCB, but with a carefully chosen confidence parameter that optimally balances
the risk of failing confidence intervals against the cost of excessive optimism. 
\\[0.2cm]

\end{abstract}

\section{Introduction}

Finite-armed bandits are the simplest and most well-studied reinforcement learning setting where an agent must
carefully balance exploration and exploitation in order to act well. This topic has seen an explosion of research
over the past half-century, perhaps starting with the work by \cite{Rob52}. 
While early researchers focussed on asymptotic results \citep[and others]{LR85} or the Bayesian setting \citep{BJK56,Git79},
recently the focus has shifted towards optimising finite-time frequentist guarantees and empirical performance.
Despite the growing body of research there are still fundamental open problems, one of which I now close.

I study the simplest setting with $K$ arms and a subgaussian noise model. 
In each time step $t$ the learner chooses an action $I_t \in \set{1,\ldots,K}$ and receives a reward $\mu_{I_t} + \eta_t$ where 
$\mu_i$ is the unknown expected reward of arm $i$ and the noise term $\eta_t$ is sampled from some $1$-subgaussian distribution that may depend on $I_t$.
For notational convenience assume throughout that
$\mu_1 > \mu_2 \geq \cdots \geq \mu_K$ and define $\Delta_i = \mu_1 - \mu_i$ to be the gap between
the expected means of the $i$th arm and the optimal arm.\footnote{This assumes the existence of a unique optimal arm, which is for mathematical convenience
only. All regret bounds will hold with natural obvious modifications if multiple optimal arms are present.}
The pseudo-regret of a strategy $\pi$ is the difference between the expected rewards
that would be obtained by the omnipotent strategy that always chooses the best arm and the expected rewards obtained by $\pi$.
\eq{
R^\pi_\mu(n) = n \mu_1 - \E\left[\sum_{t=1}^n \mu_{I_t}\right]\,,
}
where $n$ is the horizon, $I_t$ is the action chosen at time step $t$ and the expectation is taken with respect to the actions of the algorithm and the random rewards.
There are now a plethora of algorithms with strong regret guarantees, the simplest of which is 
the \textit{Upper Confidence Bound} (UCB) algorithm by \cite{AGR95,KR95} and \cite{ACF02}.\footnote{\cite{AGR95} and \cite{KR95} both proved asymptotic results
for algorithms based on upper confidence bounds, while \cite{ACF02} focussed on finite-time bounds.}
It satisfies
\eqn{
\label{eq:ucb} R^{\text{ucb}}_\mu(n) \in O\left(\sum_{i=2}^K \frac{1}{\Delta_i} \log(n)\right)\,.
}
This result is known to be asymptotically order-optimal within a class of reasonable algorithms \citep{LR85}.
But there are other measures of optimality. When one considers the worst-case regret, it can be shown
that
\eq{
\sup_\mu R^{\text{ucb}}_\mu(n) \in \Omega\left(\sqrt{nK \log n}\right)\,.
}
Quite recently it was shown by \cite{AB09} that a modified version of UCB named MOSS enjoys a worst-case regret of
\eq{
\sup_\mu R^{\text{moss}}_\mu(n) \in O\left(\sqrt{nK}\right)\,,
}
which improves on UCB by a factor of order $\sqrt{\log n}$ and matches up to constant factors the lower bound given by \cite{ACFS95}.
Unfortunately MOSS is not without its limitations. 
Specifically, one can construct regimes where the problem-dependent regret of MOSS is much worse than UCB.
The improved UCB algorithm by \cite{AO10} bridges most of the gap.
It satisfies a problem dependent regret that looks similar to \cref{eq:ucb} and a worst-case regret of
\eq{
\sup_\mu R^{\text{improved ucb}}_\mu(n) \in O\left(\sqrt{nK \log K}\right)\,,
}
which is better than UCB, but still suboptimal. Even worse, the algorithm is overly complicated and empirically hopeless. 
Thompson sampling, originally proposed by \cite{Tho33}, has gained enormous popularity due 
to its impressive empirical performance \citep{CL11} and recent theoretical guarantees \citep[and others]{KKM12,KKM13,AG12,AG12b}. 
Nevertheless, it is known that when a Gaussian prior is used, it also suffers an $\Omega(\sqrt{nK \log K})$ regret in the worst-case \citep{AG12}.

My contribution is a new algorithm called \textit{Optimally Confident UCB} (OCUCB), as well as theoretical analysis showing that 
\eq{
\sup_\mu R_\mu^{\text{ocucb}}(n) &\in O\left(\sqrt{Kn}\right) & 
R_\mu^{\text{ocucb}}(n) &\in O\left(\sum_{i=2}^K \frac{1}{\Delta_i} \log \left(\frac{n}{H_i} \right) \right) &
H_i &= \sum_{j=1}^K \min\set{\!\frac{1}{\Delta_i^2}, \frac{1}{\Delta_j^2}\!}\,.
}
The new algorithm is based on UCB, but uses a carefully chosen confidence parameter
that correctly balances the risk of failing confidence intervals against the cost of excessive optimism.
In contrast, UCB is too conservative, while MOSS is sometimes not conservative enough. The theoretical results are supported by experiments 
showing that OCUCB typically outperforms existing approaches (\cref{app:exp}).
Besides this I also present a kind of non-asymptotic problem dependent lower bound that almost matches the upper bound (\cref{sec:lower}).

\newpage
\section{Notation, Algorithm and Theorems}

\begin{wrapfigure}[9]{r}{8.8cm}
\vspace{-0.6cm}
\begin{minipage}{8.6cm}
\begin{algorithm}[H]
\KwIn{$K$, $n$, $\alpha$, $\psi$}
Choose each arm once \\
\For{$t \in K+1,\ldots,n$} {
Choose $\displaystyle I_t = \argmax_i \hat \mu_i(t) + \bonus{i}{t}$
}
\caption{Optimally Confident UCB}\label{alg:ocucb}
\end{algorithm}
\end{minipage}
\end{wrapfigure}
Let $\hat \mu_{i,s}$ be the empirical estimate of the reward of arm $i$ based on the first $s$ samples from arm $i$
and $\hat \mu_i(t)$ be the empirical estimate of the reward of arm $i$ based on the samples observed until time step $t$ (non-inclusive).
Define $T_i(t)$ to be the number of times arm $i$ has been chosen up to (not including) time step $t$.
The algorithm accepts as parameters the number of arms, the horizon, and two tunable 
variables $\alpha > 2$ and $\psi \geq 2$. The function $\logp$ is defined by $\logp(x) = \max\set{1, \log(x)}$.
A table of notation is available in \cref{app:notation}.


\begin{theorem}\label{thm:prob-dep}
If $\Delta_K \leq 1$ and $\alpha > 2$ and $\psi \geq 2$, then there exists a constant $C_1(\alpha, \psi)$ depending only on $\alpha$ and $\psi$ such that
\eq{
R_\mu^{\text{ocucb}}(n) \leq \sum_{i=2}^K \frac{C_1(\alpha, \psi)}{\Delta_i} \logp \left(\frac{n}{H_i}\right) &
& H_i &= \sum_{j=1}^K \min\set{\frac{1}{\Delta_i^2}, \frac{1}{\Delta_j^2}}\,.
}
\end{theorem}

\begin{theorem}\label{thm:prob-ind}
If $\Delta_K \leq 1$ and $\alpha > 2$ and $\psi \geq 2$, then there exists a constant $C_2(\alpha, \psi)$ depending only on $\alpha$ and $\psi$ such that
\eq{
\sup_\mu R_\mu^{\text{ocucb}}(n) \leq C_2(\alpha, \psi) \sqrt{nK}\,.
}
\end{theorem}

I make no effort to reduce the constants appearing in the regret bounds and for this reason they are left unspecified. 
Instead, I focus on maximising the range of the tunable parameters for which the algorithm is provably
order-optimal, both asymptotically and in the worst-case. The functions $C_1$ and $C_2$ have a complicated structure, but satisfy
\eq{
\forall i \in \set{1, 2} \qquad \lim_{\alpha \to \infty} C_i(\alpha,\psi) = \infty 
\quad\text{ and }\quad \lim_{\alpha \searrow 2} C_i(\alpha,\psi) = \infty 
\quad\text{ and }\quad \lim_{\psi \to\infty} C_i(\alpha,\psi) = \infty\,.
}
It is possible to improve the range of $\psi$ to $\psi > 1$ rather than $\psi \geq 2$, but this would complicate an already complicated proof.
The algorithm is very insensitive to $\psi$ and $\alpha = 3$ led to consistently excellent performance. A preliminary sensitivity analysis may be
found in \cref{app:exp}.
Both theorems depend on the assumption that $\Delta_K \leq 1$. The assumption can be relaxed
without modifying the algorithm, and with an additive penalty of $O(\sum_{i=2}^K \Delta_i)$ on the regret. 
This is due to the fact that any reasonable algorithm must choose each arm at least once.

The main difficulty in proving Theorems \ref{thm:prob-dep} and \ref{thm:prob-ind} is that the exploration bonus is simultaneously quite small and 
negatively correlated with $t$,
while for UCB it is positively correlated. A consequence is that the analysis must show that $t$ does not get too large relative to $T_1(t)$ since otherwise
the exploration bonus for the optimal arm may become too small.

\subsection{The Near-Correctness of a Conjecture}\label{sec:conjecture}

It was conjectured by \cite{BC12} that the optimal regret might be
\eqn{
\label{eq:conj} R_\mu^{\text{optimal?}}(n) \lesssim \sum_{i=2}^K \frac{1}{\Delta_i} \logp\left(\frac{n}{H}\right)\,,
}
where $H = \sum_{i =2}^K \Delta_i^{-2}$ is a quantity that appears in the best-arm identification literature \citep{BMS09,aB10,JMNB13}. 
Unfortunately this result is not attainable.
Assume a standard Gaussian noise model and let $\mu_1 = 1/2$ and $\mu_2 = 1/2 - 1/K$ and $\mu_i = 0$ for $i > 2$, which 
implies that $H = 4(K - 2) + K^2 \geq n = K^2$. Suppose $\pi$ is some policy satisfying 
$R^\pi_\mu(n) \in o(K \log K)$, which must be true for any policy witnessing \cref{eq:conj}. Then
\eq{
\min_{i > 2} \E\left[T_i(n+1)\right] \in o(\log K)\,.
}
Let $i = \argmin_{i > 2} \E\left[T_i(n+1)\right]$ and
define $\mu'$ to be equal to $\mu$ except for the $i$th coordinate, which has $\mu'_i = 1$.
Let $I = \ind{T_i(n+1) \geq n / 2}$ and let $\mathbb P$ and $\mathbb P'$ be measures on the space of outcomes induced by the interaction between $\pi$ and
environments $\mu$ and $\mu'$ respectively. Then for all $\epsilon > 0$,
\eq{
R^\pi_{\mu}(n) + R^\pi_{\mu'}(n) 
&\geq \frac{n}{2} \left(\P{I = 1} + \Pp{I = 0}\right) \sr{(a)}\geq \frac{n}{4} \exp\left(-\KL(\mathbb P, \mathbb P')\right) \\ 
&\sr{(b)}= \frac{K^2}{4} \exp\left(-\frac{\E\left[ T_i(n+1)\right]}{2}\right) \in \omega(K^{2 - \epsilon})\,,
}
where (a) follows from Lemma 2.6 by \cite{Tsy08} and (b) by computing the KL divergence between $\mathbb P$ and $\mathbb P'$, which follows along
standard lines \citep{ACFS95}.
By the assumption on $R^\pi_\mu(n)$ and for suitably small $\epsilon$ we have 
\eq{
R^\pi_{\mu'}(n) \in \omega(K^{2 - \epsilon})\,.
}
But this
cannot be true for any policy satisfying \cref{eq:conj} or even \cref{eq:ucb}. Therefore the conjecture is not true.
For the example given, $R^\pi_\mu(n) \in \Omega(K \log K)$ is necessary for any policy with sub-linear regret in $\mu'$, which
matches the regret given in \cref{thm:prob-dep}.

More intuitively, if \cref{eq:conj} were true, then the existence of a single barely suboptimal arm would significantly improve the regret relative to
a problem without such an arm, which
does not seem very plausible. The bound of \cref{thm:prob-dep}, on the other hand, depends
less heavily on the smallest gap and more on the number of arms that are nearly optimal.
There are situations where the conjecture does hold.
Specifically, when $H_i = H$, which is often approximately true (eg., if all suboptimal arms have the same gap, but this is not the only case).
I believe the bound given in \cref{thm:prob-dep} is essentially the right form of the regret. Matching lower bounds are given
in specific cases in \cref{sec:lower} along with a generally applicable lower bound that is fractionally suboptimal.

\section{\texorpdfstring{Proof of \cref{thm:prob-dep}}{Proof of Theorem}}

The proof is separated into four components. First I introduce some new notation and basic algebraic results that will hint
towards the form of the regret. I then derive the required concentration results showing that the empirical estimates of the means
lie sufficiently close to the true values. These are used to define a set of failure events that occur with low probability.
Then the number of times a suboptimal arm is pulled is bounded under the assumption that a failure event does not occur.
Finally all components are combined with a carefully chosen regret decomposition.
Throughout the proof I introduce a number of non-negative constants denoted by $\gamma, c_{\gamma}, c_1,c_2,\ldots,c_{11}$ that must satisfy certain constraints, 
which are listed and analysed in \cref{app:constants}.
Readers wishing to start with a warm-up may enjoy reading \cref{app:almost} where I give a simple and practical algorithm with the same
regret guarantees as improved UCB, but with an easy proof relying only on existing techniques and a well-chosen regret decomposition.

\subsection*{Part 0: Setup}
I start by defining some new quantities.
\eqn{
\label{def:delta}
\delta_T &= \min\set{\frac{1}{2}, \frac{c_6}{n} \sum_{i=1}^K \min\set{u_i, T}} &
u_i &= u_{\Delta_i} & 
u_\Delta &= \frac{c_9}{\Delta^2} \log\left(\frac{c_{10}}{\delta_{u_\Delta}}\right)\,. 
}
where $c_6$, $c_9$ and $c_{10}$ are constants to be chosen subsequently (described in \cref{app:constants}).
A convenient (and slightly abbusive) notation is $\delta_\Delta = \delta_{u_\Delta}$.
It is easy to check that $u_\Delta$ and $\delta_\Delta$ are monotone non-increasing.
Note that these definitions are all dependent, so the quantities must be extracted by staring at the relations. We shall gain a better
understanding of $u_\Delta$ and $\delta_\Delta$ later when analysing the regret. For now it is best to think of
$u_\Delta$ as a $(1 - \delta_\Delta)$-probability bound on the number of times a $\Delta$-suboptimal arm will be pulled. 
The following inequalities follow from straightforward algebraic manipulation.

\begin{lemma}\label{lem:delta3}
$\displaystyle \sum_{i=2}^K \Delta_i u_i \leq c_9\left(1 + \log\left(c_{10}\right)\right) \sum_{i=2}^K \frac{1}{\Delta_i} \logp\left(\frac{n}{H_i}\right)$.
\end{lemma}

\begin{lemma}\label{lem:delta}
$\delta_T \leq \delta_{T+1}$ and if $T \leq S$, then $\delta_S \leq \delta_T \cdot S/T$.
\end{lemma}

\begin{lemma}\label{lem:delta2}
Let $\gamma \in (1, \alpha/2)$ and $c_\gamma, c_5$ be as given in \cref{app:constants} and define $\tilde \delta_\Delta$ by
\eqn{
\label{def:tildedelta}
\tilde \delta_\Delta &= c_\gamma \left(\delta_\Delta + \sum_{k=0}^{k^*-1} \delta_{\gamma^{k+1}}\right) &
k^* &= \min\set{k : \gamma^{k+1} \geq \frac{1}{\Delta^2}\log\frac{1}{\delta_\Delta}}
}
Then
$\displaystyle \tilde \delta_\Delta 
\leq \frac{c_5}{n}\left(\sum_{i: u_i \geq u_\Delta} u_\Delta + \sum_{i:u_i < u_\Delta} u_i \logp\left(\frac{u_\Delta}{u_i}\right)\right)$.
\end{lemma}

\iflong
\begin{proof}
Let $k_i = \min\set{k : \gamma^{k+1} \geq u_i}$. Then
\eq{
\sum_{k=0}^{k^* - 1} \delta_{\gamma^{k+1}}
&= \frac{c_6}{n} \sum_{i=1}^K \sum_{k=0}^{k^* - 1} \min\set{u_i, \gamma^{k+1}} 
= \frac{c_6}{n} \sum_{i=1}^K \left(\sum_{k=0}^{\min\set{k_i - 1, k^* - 1}} \gamma^{k+1} + \sum_{k=k_i}^{k^*} u_i\right) \\
&\leq \frac{2\gamma c_6}{n(\gamma - 1)} \sum_{i=1}^K \left(\min\set{u_i, u_\Delta} + \ind{u_i < u_\Delta} u_i \logp\left(\frac{u_i}{u_\Delta}\right)  \right)
}
The result by adding $\delta_\Delta$ and naive bounding. 
\end{proof}
\fi

\subsection*{Part 1: Concentration}
\begin{lemma}\label{lem:conc}
Let $X_1, X_2,\ldots$ be sampled i.i.d.\ from some $1$-subgaussian distribution and let $\hat \mu_t = \sum_{s=1}^t X_s / t$ be the empirical
mean based on the first $t$ samples. Suppose $\beta \geq 1$. Then for all $\Delta > 0$,
\eq{
\P{\exists t : |\hat \mu_t| \geq \sqrt{\frac{2\gamma \beta}{t} \log \frac{1}{\delta_t}} + c_4 \Delta} \leq 
\tilde \delta_\Delta 2^{-\beta}\,.
}
\end{lemma}

The proof may be found in \cref{app:lem:conc} and is based on a peeling argument combined with 
Doob's maximal inequality (e.g., as was used by \cite{AB09,Bub10} and elsewhere). 
Define $\beta_{i,\Delta} \geq 1$ by
\eqn{
\label{def:beta}
\beta_{i,\Delta} = \min\set{\beta \geq 1 : (\forall t)\,\, |\hat \mu_{i,t} - \mu_i| \leq \sqrt{\frac{2\gamma \beta}{t} \log\frac{1}{\delta_t}} +c_4 \Delta }\,. 
}
Note that for fixed $\Delta$ the random variables $\beta_{i,\Delta}$ with $i \in \set{1,\ldots,K}$ are (mutually) independent.
Furthermore, if $i$ is fixed, then $\beta_{i,\Delta}$ is non-increasing as $\Delta$ increases.

\begin{lemma}\label{lem:beta_one}
$\P{\beta_{i,\Delta} > 1} \leq \tilde \delta_\Delta$ and $\E[\beta_{i,\Delta} - 1] \leq 2\tilde \delta_\Delta$.
\end{lemma}

\begin{lemma}\label{lem:average}
For all $\Delta,\Delta' > 0$,
$\displaystyle \P{\sum_{i : \Delta_i \geq \Delta'} \beta_{i,\Delta} u_i \Delta_i \geq 2\sum_{i : \Delta_i \geq \Delta'} u_i \Delta_i} \leq 2\tilde \delta_\Delta$.
\end{lemma}

\begin{lemma}\label{lem:uniform1}
$\displaystyle \P{\exists T : \sum_{i : \beta_{i,\Delta} = 1} \min\set{u_i, T} \leq \frac{1}{5} \sum_{i=1}^K \min\set{u_i, T}} \leq 24\tilde\delta_\Delta$.
\end{lemma}

\begin{lemma}\label{lem:uniform2}
$\displaystyle \P{\exists T : \sum_{j=1}^K \beta_{j,\Delta} \min\set{u_j, T} \geq 67 \sum_{j=1}^K \min\set{u_j, T}} \leq 13\tilde \delta_\Delta$.
\end{lemma}

Lemma \ref{lem:beta_one} follows from Lemma \ref{lem:conc}. Lemma \ref{lem:average} follows from Lemma \ref{lem:beta_one} via Markov's inequality.
The proofs of Lemmas \ref{lem:uniform1} and \ref{lem:uniform2} are given in \cref{app:uniform},
with the only difficulty being the uniformity over $T$ and because a naive application of the union bound would lead to an unpleasant
dependence on $K$ or $n$. Both results would follow trivially from Markov's inequality for fixed $T$.

\subsection*{Part 2: Failure Events}

For each $\Delta \geq 0$, define $F_\Delta \in \set{0,1}$ to be the event that one of the following does not hold: 
\hfill \refstepcounter{equation}\label{def:failure} (\theequation) \\
\begin{flalign*}
&(\text{C1}): \beta_{1,\Delta} = 1  & \\
&(\text{C2}): \sum_{i : \Delta_i \geq c_8 \Delta} \beta_{i,\Delta} u_i \Delta_i \leq 2\sum_{i : \Delta_i \geq c_8 \Delta} u_i \Delta_i  & \\
&(\text{C3}): \forall T : \sum_{i : \beta_{i,\Delta} = 1}\min\set{u_i, T} \geq \frac{1}{5} \sum_{i=1}^K \min\set{u_i, T} & \\ 
&(\text{C4}): \forall T : \sum_{j=1}^K \beta_{j,\Delta} \min\set{u_j, T} \leq 67\sum_{j=1}^K \min\set{u_j, T} \,. & 
\end{flalign*}
By Lemmas \ref{lem:beta_one}, \ref{lem:average}, \ref{lem:uniform1} and \ref{lem:uniform2} 
we have
$\P{F_\Delta} \leq (2 + 2 + 24 + 13)\tilde \delta_\Delta = 41 \tilde \delta_\Delta$. Define
\eqn{
\label{def:failure2} \tilde\Delta = \sup \set{\Delta : F_{\Delta} = 1}\,.
}
From the definition of $\beta_{i,\Delta}$ we have that $F_{\Delta} = 1$ for all $\Delta < \tilde\Delta$ and $F_{\Delta} = 0$ for 
all $\Delta \geq \tilde\Delta$. We will shortly see that the algorithm will quickly eliminate arms with gaps larger than $\tilde\Delta$, while arms
with gaps smaller than $\tilde\Delta$ may be chosen linearly often.

\subsection*{Part 3: Bounding the Pull Counts}

This section contains the most important component of the proof, which is 
bounding $T_j(n+1)$ for arms $j$ with $\Delta_j$ larger than a constant factor times $\tilde \Delta$.
I abbreviate $\beta_i = \beta_{i,\tilde \Delta}$ for this part. 
The proof is rather involved, so I try to give some intuition.
We need to show that if $T_j(t) = \ceil{\beta_j u_j}$, then the error of the empirical estimate of 
the return of arm $j$ and arm $1$ are both around $\Delta_j$ and that the bonus for arm $j$ is also not significant.
To do this we will show that the pull-counts of near-optimal arms are at least a constant proportion of $T_j(t)$ and it
is this that presents the most difficulty.

\begin{lemma}\label{lem:opt1}
Let $t$ be some time step and $i,j$ be arms such that: \\
\hspace{0.2cm} 1. $\beta_i = 1$ \hspace{0.4cm} 
2. $c_2 \beta_jT_i(t) \leq T_j(t)$ \hspace{0.4cm}
3. $\psi n/t \geq 1/\delta_{T_i(t)}$ \hspace{0.4cm}
4. $c_1 T_i(t) \leq \min\set{u_i, u_{\tilde \Delta}}$. Then $I_t \neq j$.
\end{lemma}

\begin{proof}
Arm $j$ is not played if arm $i$ has a larger index. 
\eq{
&\hat \mu_i(t) + \bonus{i}{t} 
\sr{(a)}\geq \mu_i + \bonus{i}{t} - \conf{i}{t}{\delta_{T_i(t)}}{} - c_4 \tilde \Delta  \\
&\sr{(b)}\geq \mu_j + \bonus{i}{t} - \conf{i}{t}{\delta_{T_i(t)}}{} - \Delta_i - c_4 \tilde\Delta  \\
&\sr{(c)}\geq \hat \mu_j(t) + \bonus{i}{t} - \conf{i}{t}{\delta_{T_i(t)}}{} - \conf{j}{t}{\delta_{T_j(t)}}{\beta_j} - \Delta_i - 2c_4 \tilde\Delta \\ 
&\sr{(d)}\geq \hat \mu_j(t) + \bonus{i}{t} - \left(\sqrt{2\gamma} + \sqrt{\frac{2\gamma}{c_2}}\right) \sqrt{\frac{1}{T_i(t)} \log\left(\frac{1}{\delta_{T_i(t)}}\right)} - \Delta_i - 2c_4 \tilde\Delta \\
&\sr{(e)}\geq \hat \mu_j(t) + \bonus{j}{t} + \left(1 - \sqrt{\frac{1}{c_2}}\right) \bonus{i}{t} \\
 &\qquad\qquad - \left(\sqrt{2\gamma} + \sqrt{\frac{2\gamma}{c_2}}\right) \sqrt{\frac{1}{T_i(t)} \log\left(\frac{1}{\delta_{T_i(t)}}\right)} - \Delta_i - 2c_4\tilde \Delta \\
&\sr{(f)}\geq \hat \mu_j(t) + \bonus{j}{t} \\
  &\qquad\qquad + \left(\sqrt{\alpha} - \sqrt{\frac{\alpha}{c_2}} - \sqrt{2\gamma} - \sqrt{\frac{2\gamma}{c_2}}\right)\sqrt{\frac{1}{T_i(t)} \log\left(\frac{1}{\delta_{T_i(t)}}\right)} - \Delta_i - 2c_4\tilde \Delta \\
&\sr{(g)}\geq \hat \mu_j(t) + \bonus{j}{t} + \max\set{
    \sqrt{\frac{c_1}{u_i} \log\left( \frac{1}{\delta_{u_i}}\right)},
    \sqrt{\frac{c_1}{u_{\tilde\Delta}} \log\left( \frac{1}{\delta_{\tilde\Delta}}\right)}
} - \Delta_i - 2c_4 \tilde\Delta \\ 
&\sr{(h)}> \hat \mu_j(t) + \bonus{j}{t} \,,
}
where (a) follows since $\beta_i = 1$ and $F_{\tilde\Delta} = 0$,
(b) since $\mu_i = \mu_1 - \Delta_i \geq \mu_j - \Delta_i$,
(c) since $F_{\tilde \Delta} = 0$,
(d) and (e) since $c_2 \beta_j T_i(t) \leq T_j(t)$ and because $\delta_T$ is monotone non-decreasing,
(f) since $\psi n/t \geq 1/\delta_{T_i(t)}$ is assumed,
(g) from the constraint on $c_2$ \CONref{C:c2} and because $T_i(t) \leq \min\set{u_i / c_1, u_{\tilde\Delta} / c_1}$,
(h) from the constraint on $c_1$ \CONref{C:c1} and from $\max\set{x,y} \geq x/2+y/2$.
\end{proof}

\begin{lemma}\label{lem:opt2}
Let $t$ be some time step and $j$ be an arm such that: \\
1. $\Delta_j \geq c_8\tilde \Delta$ \hspace{0.4cm}
2. $T_j(t) = \ceil{\beta_j u_j}$ \hspace{0.4cm}
3. $c_2 \beta_j T_1(t) \geq T_j(t)$ or $T_1(t) \geq u_{\tilde\Delta} / c_1$ \hspace{0.4cm}
4. $\psi n/t \leq c_7 \beta_j /\delta_{u_j}$. 
Then $I_t \neq j$.
\end{lemma}

\begin{proof}
Using a similar argument as in the proof of the previous lemma.
\eq{
&\hat \mu_1(t) + \bonus{1}{t} 
\sr{(a)}\geq \mu_1 + \bonus{1}{t} - \conf{1}{t}{\delta_{T_1(t)}}{} - c_4 \tilde\Delta \\
& \sr{(b)}=\mu_j + \bonus{1}{t} - \conf{1}{t}{\delta_{T_1(t)}}{} + \Delta_j - c_4 \tilde\Delta \\
& \sr{(c)}\geq \hat \mu_j(t) + \bonus{1}{t} - \conf{1}{t}{\delta_{T_1(t)}}{} - \conf{j}{t}{\delta_{T_j(t)}}{\beta_j} + \Delta_j - 2c_4 \tilde\Delta \\
& \sr{(d)}\geq \hat \mu_j(t) + \bonus{1}{t} \\ &\qquad - 
\max\set{
  \sqrt{\frac{2\gamma c_2 }{u_j} \log\left(\frac{c_2}{\delta_{u_j}}\right)},\,
  \sqrt{\frac{2\gamma c_1 }{u_{\tilde\Delta}} \log\left(\frac{c_1}{\delta_{u_{\tilde\Delta}}}\right)}
} -
\sqrt{\frac{2\gamma}{u_j} \log\left(\frac{1}{\delta_{u_j}}\right)} + \Delta_j - 2c_4 \tilde\Delta \\
& \sr{(e)}\geq \hat \mu_j(t) + \bonus{1}{t} + \Delta_j /2 - (2+c_3)c_4 \tilde\Delta \\
& \sr{(f)}\geq \hat \mu_j(t) + \bonus{j}{t} - \bonus{j}{t} + \Delta_j / 2 - (2+c_3)c_4 \tilde\Delta \\
& \sr{(g)}\geq \hat \mu_j(t) + \bonus{j}{t} - \sqrt{\frac{\alpha}{T_j(t)} \log\left(\frac{c_7 \beta_j}{\delta_{u_j}}\right)} + \Delta_j/2 - (2+c_3)c_4 \tilde\Delta \\
& \sr{(h)}\geq \hat \mu_j(t) + \bonus{j}{t} - \sqrt{\frac{\alpha}{u_j \beta_j} \log\left(\frac{c_7 \beta_j}{\delta_{u_j}}\right)} + \Delta_j /2 - (2+c_3)c_4 \Delta_j / c_8 \\
& \sr{(i)}> \hat \mu_j(t) + \bonus{j}{t} \,.
}
where (a) follows since $\beta_1 = 1$ and because $F_{\tilde\Delta} = 0$,
(b) since $\mu_1 = \mu_j + \Delta_j$,
(c) by the definition of $\beta_j$ and because $F_{\tilde\Delta} = 0$ does not hold,
(d) by the assumption that $c_2 \beta_j T_1(t) \geq T_j(t)$ or $c_1 T_1(t) \geq u_{\tilde \Delta}$ and because $T_j(t) = \ceil{\beta_j u_j}$ and Lemma \ref{lem:delta},
(e) by the constraints on $u_j$ \CONreft{C:uj-1}{C:uj-2} and on $c_1$ \CONref{C:c1-2}.
(f) is trivial,
(g) since we assumed $\psi n/t \leq c_7/\delta_{u_j}$,
(h) since $\Delta_j \geq c_8 \tilde\Delta$,
(i) from constraints \CONref{C:uj-3} and \CONref{C:c3}. 
\end{proof}

\begin{lemma}\label{lem:opt3}
If $\Delta_j \geq c_8 \tilde\Delta$, then $T_j(n+1) \leq \ceil{\beta_j u_j}$.
\end{lemma}

\begin{proof}
We need two results for all $t$:
\begin{enumerate}
\item $\beta_i = 1 \implies t \leq \psi n \delta_{T_i(t)}$ or $c_1 T_i(t) \geq \min\set{u_i, u_{\tilde \Delta}}$.
\item $c_7 \beta_j t \geq \psi n \delta_{T_j(t)}$
\end{enumerate}
Both are trivial for $t = K+1$. Assume (a) and (b) hold for all $s < t$.
Then
\eq{
c_7 \beta_j t 
&\sr{(a)}\geq c_7 \beta_j \sum_{i=1}^K T_i(t) 
\sr{(b)}\geq c_7 \beta_j \sum_{i : \beta_i = 1} T_i(t)
\sr{(c)}\geq c_7 \beta_j \sum_{i : \beta_i = 1} \min\set{\frac{u_i}{c_1}, \frac{u_{\tilde\Delta}}{c_1}, \frac{T_j(t) - 1}{c_2 \beta_j}} \\
&\sr{(d)}\geq \frac{c_7 \beta_j}{c_1 + c_2} \sum_{i : \beta_i = 1} \min\set{u_i, u_{\tilde\Delta}, \frac{T_j(t) - 1}{\beta_j}} 
\sr{(e)}\geq \frac{c_7 \beta_j}{c_1 + c_2} \sum_{i : \beta_i = 1} \min\set{u_i, \frac{T_j(t) - 1}{\beta_j}}  \\
&\sr{(f)}\geq \frac{c_7}{c_1 + c_2} \sum_{i : \beta_i = 1} \min\set{u_i, T_j(t) - 1} 
\sr{(g)}\geq \frac{c_7}{5(c_1 + c_2)}\left(\sum_{i=1}^K \min\set{u_i, T_j(t)} - K\right) \\ 
&\sr{(h)}\geq \frac{c_7 n}{5c_6(c_1 + c_2)} \left(\delta_{T_j(t)} - \frac{c_6 K}{n}\right) 
\sr{(i)}\geq \frac{c_7 n}{10c_6(c_1 + c_2)} \delta_{T_j(t)}
\sr{(j)}= \psi n \delta_{T_j(t)}\,,
}
where (a) and (b) are trivial. (c) follows from Lemma \ref{lem:opt1} and the assumption that 1.\ and 2.\ hold for all $s < t$.
(d) follows since $c_1, c_2 \geq 1$.
(e) since $T_j(t) / \beta_j \leq u_j \leq u_{\tilde \Delta}$ by Lemma \ref{lem:opt2} and the assumption that $\Delta_j \geq c_8 \tilde\Delta$.
(f) is trivial. (g) follows from condition (C3). 
(h) follows from the definition of $\delta_{T_j(t)}$. 
(i) by naively bounding $\delta_{T_j(t)}$ and (j) by the definition of $c_7$ \CONref{C:c7}.
Therefore 2.\ holds also for $t$.
Now suppose $\beta_i = 1$ and $c_1 T_i(t) < \min\set{u_i, u_{\tilde\Delta}}$. Then by Lemma \ref{lem:opt1} we have for any $k$
that $T_k(t) \leq c_2 \beta_k T_i(t) + 1$. If $\Delta_k \geq c_8\tilde\Delta$, then $T_k(t) \leq \beta_k u_k + 1$.
On the other hand, if $\Delta_k < c_8 \tilde\Delta$. Then
$T_k(t) \leq c_2 \beta_k T_i(t) + 1 \leq c_2 \beta_k u_{\tilde \Delta} + 1 \leq c_2 c_8^3 \beta_k u_k + 1$.
Therefore
\eq{
t 
&\sr{(a)}= 1 + \sum_{k=1}^K T_k(t)
\sr{(b)}\leq K + 1 + c_2 c_8^3 \sum_{k=1}^K \beta_k \min\set{T_i(t), u_k} \\
&\sr{(c)}\leq K + 1 + 67 c_2 c_8^3 \sum_{k=1}^K \min\set{T_i(t), u_k}
\sr{(d)}= K + 1 + \frac{67c_2 c_8^3}{c_6} n\delta_{T_i(t)} \\
&\sr{(e)}\leq \frac{134 c_2 c_8^3}{c_6} n\delta_{T_i(t)} 
\sr{(f)}= \psi n \delta_{T_i(t)}\,,
}
where (a) is trivial.
(b) follows from the reasoning above the display and naively choosing largest possible constant.
(c) follows from condition (C4) in the definition of the failure event.
(d) by substituting the definition of $\delta_{T_i(t)}$.
(e) by naively bounding $\delta_{T_i(t)}$ and noting that $T_i(t) \geq 1$.
(f) is the definition of $c_6$ \CONref{C:c6}.
Therefore 1.\ and 2.\ hold for all $t$ and so by Lemmas \ref{lem:opt1} and \ref{lem:opt2} we have $T_j(t) \leq \ceil{\beta_j u_j}$ as required.
\end{proof}

\subsection*{Part 4: Regret Decomposition}
Let $R = n \mu_1 - \sum_{t=1}^n \mu_{I_t}$ be the pseudo-regret (this is a random variable because there is no expectation on $I_t$). 
From the previous section, if $\Delta_i \geq c_8 \tilde\Delta$, then $T_i(n+1) \leq \ceil{\beta_{i,\tilde \Delta} u_i}$.
Therefore
\eq{
R 
&\leq c_8 n\tilde\Delta\cdot \ind{c_8 \tilde\Delta \geq\Delta_2} + \sum_{i : \Delta_i \geq c_8 \tilde\Delta} \Delta_i \ceil{\beta_{i,\tilde \Delta} u_i}  
\leq c_8 n\tilde\Delta\cdot \ind{c_8 \tilde\Delta \geq\Delta_2} + 3\sum_{i=2}^K \Delta_i u_i\,,
}
which follows from the definition of the failure event (C2) and naive simplification.
Therefore
\eqn{
\label{eq:decomp}
R^{\text{ocucb}}_\mu(n) 
= \E R 
\leq 3\sum_{i=2}^K \Delta_i u_i + c_8 n \E\left[\tilde\Delta \cdot \ind{c_8 \tilde\Delta \geq \Delta_2}\right]\,.
}
All that remains is to bound the expectation. Starting with an easy lemma.

\begin{lemma}\label{lem:int}
$\displaystyle \int^{\infty}_{\Delta_i} u_\Delta d\Delta \leq \frac{9}{\Delta_i}\left(2 + \log\left(\frac{c_{10}}{\delta_{\Delta_i}}\right)\right)$
and
$\displaystyle \int^{\Delta_i}_0 u_i \log\left(\frac{u_{\Delta}}{u_i}\right) d\Delta \leq 2\Delta_i u_i$.
\end{lemma}

\begin{proof}
By straight-forward calculus and Lemma \ref{lem:delta4} in \cref{app:tech}.
\eq{
\int^\infty_{\Delta_i} u_\Delta d\Delta
&= \int^\infty_{\Delta_i} \frac{c_9}{\Delta^2} \log\left(\frac{c_{10}}{\delta_\Delta}\right) d\Delta 
= \int^\infty_{\Delta_i} \frac{c_9}{\Delta^2} \log\left(\frac{c_{10}}{\delta_{\Delta_i}} \cdot \frac{\delta_{\Delta_i}}{\delta_\Delta}\right) d\Delta \\ 
&\leq \int^\infty_{\Delta_i} \frac{c_9}{\Delta^2} \log\left(\frac{c_{10}}{\delta_{\Delta_i}} \cdot \frac{\Delta^2}{\Delta_i^2} \right) d\Delta 
= \frac{c_9}{\Delta_i} \left(2 + \log\left(\frac{c_{10}}{\delta_{\Delta_i}}\right)\right)\,.
}
For the second part Lemma \ref{lem:u} gives $\displaystyle \int^{\Delta_i}_{0} u_i \log \left(\frac{u_{\Delta}}{u_i}\right) d\Delta
\leq \int^{\Delta_i}_{0} u_i \log\left(\frac{\Delta_i^2}{\Delta^2}\right) d\Delta 
= 2\Delta_i u_i$. 
\end{proof}

\begin{lemma}\label{lem:expect}
$\displaystyle n\E\left[\tilde\Delta \ind{\tilde\Delta \geq \Delta_2 / c_8} \right] \leq c_{11} \sum_{i=2}^K \Delta_i u_i$.
\end{lemma}

\begin{proof}
Preparing to use the previous lemma.
\eq{
\E\left[\tilde\Delta \ind{\tilde\Delta \geq \frac{\Delta_2}{c_8}}\right]
&\leq \frac{\Delta_2}{c_8}\P{\tilde\Delta \geq \frac{\Delta_2}{c_8}}  + \int^\infty_{\Delta_2 / c_8} \P{\tilde\Delta \geq \Delta} d\Delta \\ 
&\leq \frac{\Delta_2 \tilde \delta_{\Delta_2/c_8}}{c_8}  + \int^\infty_{\Delta_2 / c_8} \tilde \delta_{\Delta} d\Delta 
}
Bounding each term separately. First, by Lemmas \ref{lem:delta2} and \ref{lem:u} we have 
\eq{
\frac{\Delta_2 \tilde \delta_{\Delta_2/c_8}}{c_8} 
&\leq \frac{1}{n} \cdot \frac{c_5 \Delta_2}{c_8} \left(u_{\Delta_2/c_8} + \sum_{i=2}^K u_i \log\left(\frac{u_{\Delta_2/c_8}}{u_i}\right) \right) \\
&\leq \frac{1}{n} \cdot c_5 c_8 \Delta_2 \left(u_{\Delta_2} + \sum_{i=2}^K u_i \log\left(\frac{c_8^2 \Delta_i^2}{\Delta_2^2}\right) \right)
\leq \frac{1}{n} \cdot 2(1 + \log(c_8)) c_5 c_8 \sum_{i=2}^K \Delta_i u_i\,.
}
For the second term, using Lemma \ref{lem:delta2} again, as well as Lemma \ref{lem:int}
\eq{
&\int^\infty_{\Delta_2/c_8} \tilde \delta_\Delta d\Delta
\leq \frac{c_5}{n} \int^\infty_{\Delta_2/c_8} \left(\sum_{i : u_i \geq u_\Delta} u_\Delta + \sum_{i: u_i < u_\Delta} u_i \log\left(\frac{u_\Delta}{u_i}\right)\right) d\Delta \\
&\quad= \frac{c_5}{n}\left( \int^\infty_{\Delta_2/c_8} u_\Delta d\Delta + \sum_{i=2}^K \int^{\infty}_{\Delta_i} u_\Delta d\Delta + \sum_{i=2}^K \int^{\Delta_i}_{0} u_i \log\left(\frac{u_\Delta}{u_i}\right) d\Delta\right) \\
&\quad\leq \frac{c_5}{n} \left(\frac{c_9}{\Delta_2/c_8} \left(2 + \log\left(\frac{c_{10}}{\delta_{\Delta_2/c_8}}\right)\right)
  + \sum_{i=2}^K \frac{c_9}{\Delta_i} \left(2 + \log\left(\frac{c_{10}}{\delta_{\Delta_i}}\right)\right) 
  + \sum_{i=2}^K 2\Delta_i u_i\right) \\
}
The result follows by choosing
$c_{11} = 2c_5 c_8(1 + \log(c_8)) + c_5 (3c_8 \log(c_8) + 5)$.
\end{proof}

And with this we have the final piece of the puzzle. Substituting Lemma \ref{lem:expect} into \cref{eq:decomp}:
\eq{
R^{\text{ocucb}}_\mu(n) \in O\left(\sum_{i=2}^K \Delta_i u_i\right)\,.
}
Collecting the constants and applying Lemma \ref{lem:delta3} leads to
\eq{
R^{\text{ocucb}}_{\mu}(n) 
&\leq \left(3 + c_{11}\right) \sum_{i=2}^K \Delta_i u_i 
\leq c_9 \left(3 + c_{11}\right)\left(1 + \log\left(c_{10}\right)\right) \sum_{i=2}^K \frac{1}{\Delta_i} \log\left(\frac{n}{H_i}\right)\,.
}
The result is completed by choosing
$\displaystyle C_1(\alpha, \psi) = c_9 \left(3 + c_{11}\right)\left(1 + \log\left(c_{10}\right)\right)$.

\section{\texorpdfstring{Proof of \cref{thm:prob-ind}}{Proof of Theorem}}

The proof follows exactly as the proof of \cref{thm:prob-dep}, but bounding the regret due to arms with $\Delta_i \leq \sqrt{K/n}$ by $\sqrt{Kn}$.
Then
\eq{
\E R^{\text{ocucb}}_\mu(n) 
&\leq \sqrt{Kn} + \sum_{i : \Delta_i > \sqrt{K/n}} \frac{C_1(\alpha,\psi)}{\Delta_i} \logp\left(\frac{n}{H_i}\right) 
\leq \sqrt{Kn} + C_1(\alpha,\psi) \sqrt{Kn}\,,
}
where the last line follows by substituting the definition of $H_i$ and solving the optimisation problem.
Finally set $C_2(\alpha,\psi) = 1 + C_1(\alpha,\psi)$.

\newcommand{\lastname}{}
\newcommand{\args}{}
\newcounter{axiscounter}
\newcommand{\name}{}
\newcommand{\defaultargs}{scaled y ticks=false,scaled x ticks=false,xshift=0.1cm,width=3.2cm,height=3.2cm,compat=newest}

\newcommand{\simpleaxis}[5][]{
\renewcommand\name{\the\numexpr\value{axiscounter}\relax} 
\ifnum1>\value{axiscounter}
\begin{axis}[name=\name,#1,axis on top=true]
\else
\begin{axis}[name=\name,#1,at=(\the\numexpr\value{axiscounter}-1\relax.south east),xshift=0.1cm,xtick={},yticklabel=\empty,xlabel=\empty,ylabel=\empty,xticklabel=\empty,axis on top=true]
\fi
  \addplot [forget plot,draw=none, stack plots=y, forget plot,smooth] table [x=0,y expr=\thisrow{#3}] {#5};
  \addplot [forget plot,draw=none,fill=black!20, stack plots=y, area legend] table [x=0,y expr=\thisrow{#4} - \thisrow{#3}] {#5} \closedcycle;
  \addplot [forget plot, stack plots=y,draw=none] table [x=0, y expr=-(\thisrow{#4})] {#5};

  \addplot[mark=none,line width=1pt,black] table[x index=0,y index=#3] {#5};
\end{axis}
\if\relax\detokenize{#2}\relax
\else
\node[anchor=south,minimum height=0.5cm] at (\name.north) {#2};
\fi
\stepcounter{axiscounter}
}

\section{Brief Experiments}\label{sec:brief-exp}

\begin{wrapfigure}[8]{r}{7cm}

\vspace{-1.3cm}
\begin{tikzpicture}[font=\scriptsize]
\begin{axis}[xtick=data,scaled x ticks=false,width=7cm,height=5.5cm,xlabel near ticks,ylabel near ticks,
  xtick={0,1},
  xlabel={$\Delta$},
  xmin=0,
  xmax=1,
  legend cell align=left,
  ylabel={Expected regret}]
    \addplot+[smooth] table[x index=0,y index=1] \dataOne;
    \addlegendentry{UCB};
    \addplot+[smooth] table[x index=0,y index=2] \dataOne;
    \addlegendentry{OCUCB};
    \addplot+[smooth] table[x index=0,y index=3] \dataOne;
    \addlegendentry{Thompson sampling};
\end{axis}
\end{tikzpicture}
\end{wrapfigure}
The graph on the right teases the worst-case performance of OCUCB relative to UCB and 
Thompson Sampling when $n = 10^4$, $K = 2$ and where $\Delta_2$ is varied. Precise details are given in \cref{app:exp} where
OCUCB is comprehensively evaluated in a variety of regimes and compared to many strategies including MOSS, AOCUCB and the finite-horizon Gittins index
strategy.

\section{Conclusions}

The \textit{Optimally Confident UCB} algorithm is the first algorithm that
simultaneously enjoys order-optimal problem-dependent and worst-case regret guarantees. The algorithm is simple, extremely efficient (see \cref{app:comp}) and
empirically superb (see \cref{app:exp}).
The main conceptual contribution is a greater understanding of how to optimally select the confidence level when designing optimistic algorithms
for solving the exploration/exploitation trade-off. 
There are some open problems.

\subsubsect{Improving Analysis and Constants}
Much effort has been made to maximise the region of the parameters $\alpha$ and $\psi$ for which order-optimal
regret is guaranteed. Unfortunately the empirical choices are not supported by minimising
the regret bound with respect to $\alpha$ and $\psi$ (which in any case would be herculean task). 
The open problem is to derive a simple proof of the main theorems for which the theoretically optimal $\alpha$ and $\psi$ are also 
practical. Along the way it should be possible to modify the index to show exact asymptotically optimality. My presumption is that this can be done by setting
$\alpha = 2$ and adding an additional $O(\log \log t)$ bonus as in KL-UCB.

\subsubsect{Anytime Algorithms}
The new algorithm is not anytime because it requires knowledge of the horizon in advance (MOSS is also not anytime, but Thompson sampling is). 
It should be possible to apply the same repeated restarting idea as was used by \cite{AO10}, but this is seldom practical. 
Instead it would be better to modify the algorithm to smoothly adapt to an increasing horizon.
As an aside, an algorithm is not necessarily worse because it needs to know the horizon in advance. An occasionally reasonable
alternative view is that such algorithms have an \textit{advantage} because they can exploit available information. There
may be cause to modify Thompson sampling (or other algorithms) so that they can also exploit a known horizon.

\subsubsect{Exploiting Low Variance}
There is also the question of exploiting low variance when the rewards are not Gaussian. Much work has been done in this setting, especially
when the rewards are bounded (Eg., the KL-UCB algorithm by \cite{Gar11,OMS11,CGMMS13} or UCB-V by \cite{AMS07}), but also more generally \citep{BCL13}.
It is not hard to believe that some of the ideas used in this paper extend to those settings (or vice versa).
Related is the question of how to trade robustness and expected regret. Merely increasing $\psi$ (or $\alpha$) will decrease the variance of OCUCB, while perhaps
retaining many of the positive qualities of the choice of confidence interval. A theoretical and empirical investigation would be interesting. 

\subsubsect{Optimal Lower Bounds}
The lower bound given in \cref{sec:lower} is very slightly suboptimal and can likely be strengthened by removing the $\log \log K$ term.
Likely the form of the statement can also be altered to emphasise the OCUCB really is making a well-justified trade-off.

\subsubsect{Extensions}
Besides the improvement for finite-armed bandits, I am hopeful that some of the techniques  
may also be generalisable to the stochastic linear (or contextual) bandit settings for which we do not yet have worst-case optimal algorithms 
(see, for example, the work by \cite{DHK08,AST11,CM12,RT10}).



\appendix

\ifcolt
\else
\bibliographystyle{plainnat}
\fi
\bibliography{all}


\section{\texorpdfstring{Proof of Lemma \ref{lem:conc}}{Proof of Theorem}} \label{app:lem:conc}

I briefly prove a maximal version of the standard concentration inequalities for i.i.d\ subgaussian random variables.
The proof is totally standard and presumably has been written elsewhere, but a reference proved elusive.
Since $X_i$ is $1$-subgaussian, by definition it satisfies
\eq{
(\forall \lambda \in \R) \qquad \E\left[\exp\left(\lambda X_i\right)\right] \leq \exp\left(\lambda^2/2\right)\,.
}
Now $X_1,X_2,\ldots$ are i.i.d.\ and zero mean, so by convexity of the exponential function $\exp(\lambda \sum_{s=1}^t X_s)$ is a sub-martingale.
Therefore if $\epsilon > 0$, then by Doob's maximal inequality 
\eqn{
\nonumber \P{\exists t \leq n : \sum_{s=1}^t X_s \geq \epsilon} 
&= \inf_{\lambda \geq 0} \P{\exists t \leq n : \exp\left(\lambda \sum_{s=1}^t X_s\right) \geq \exp\left(\lambda \epsilon\right)} \\ 
\nonumber &\leq \inf_{\lambda \geq 0} \exp\left(\frac{\lambda^2 n}{2} -\lambda \epsilon\right) \\
&= \exp\left(-\frac{\epsilon^2}{2n}\right)\,.
\label{eq:maximal}
}
Now we use the peeling argument.
\eq{
&\P{\exists t : |\hat \mu_t| \geq \sqrt{\frac{2\gamma\beta}{t} \log \frac{1}{\delta_t}} + c_4 \Delta} \\
&\sr{(a)}\leq \sum_{k=0}^{\infty} \P{\exists t \in [\gamma^k, \gamma^{k+1}] : |\hat \mu_t| \geq \sqrt{\frac{2\gamma\beta }{t} \log \frac{1}{\delta_t}} + \gamma\Delta \sqrt{2}} \\
&\sr{(b)}\leq \sum_{k=0}^{\infty} \P{\exists t \in [\gamma^k, \gamma^{k+1}] : \left|\sum_{s=1}^t X_s\right| \geq \sqrt{2\gamma\beta t \log \frac{1}{\delta_t}} + t\gamma\Delta \sqrt{2}} \\
&\sr{(c)}\leq \sum_{k=0}^{\infty} \P{\exists t \in [\gamma^k, \gamma^{k+1}] : \left|\sum_{s=1}^t X_s\right| \geq \sqrt{2\gamma\beta t \log \frac{1}{\delta_{\gamma^{k+1}}}} + t\gamma\Delta \sqrt{2}} \\
&\sr{(d)}\leq \sum_{k=0}^{\infty} \P{\exists t \leq \gamma^{k+1} : \left|\sum_{s=1}^t X_s\right| \geq \sqrt{2\gamma \cdot \gamma^k\beta  \log \frac{1}{\delta_{\gamma^{k+1}}}} + \gamma^{k+1}\Delta\sqrt{2}} \\
&\sr{(e)}\leq 2\sum_{k=0}^\infty \delta_{\gamma^{k+1}}^\beta \exp\left(-\frac{2\gamma^{2k+2} \Delta^2}{2\cdot \gamma^{k+1}}\right) \\
&\sr{(f)}\leq 2^{2- \beta}\sum_{k=0}^\infty \delta_{\gamma^{k+1}} \exp\left(-\gamma^{k+1} \Delta^2 \right) \\
}
where 
(a) follows from the union bound,
(b) follows from the definition of $\hat \mu_t$,
(c) follows since $\delta_t$ is non-decreasing,
(d) since $\Delta > 0$ and $t \geq \gamma^k$,
(e) from the maximal inequality \cref{eq:maximal} and
(f) since $\delta_t \leq 1/2$ for all $t$.
Let 
\eq{
k^* = \min \set{k : \gamma^{k+1} \geq 1/\Delta^2 \log 1/ \delta_{\Delta}}\,. 
}
Then $\gamma^{k+1} \leq u_{\Delta}$ for all $k \leq k^*$ and so
\eq{
&2^{2- \beta}\sum_{k=0}^\infty \delta_{\gamma^{k+1}} \exp\left(-\gamma^{k+1} \Delta^2 \right) \\ 
&\leq 2^{2- \beta}\sum_{k=0}^{k^*-1} \delta_{\gamma^{k+1}} \exp\left(-\gamma^{k+1} \Delta^2 \right) + 2^{2- \beta}\sum_{k=k^*}^\infty \delta_{\gamma^{k+1}} \exp\left(-\gamma^{k+1} \Delta^2 \right) \\ 
&\leq 2^{2- \beta}\sum_{k=0}^{k^*-1} \delta_{\gamma^{k+1}} + 2^{2- \beta}\sum_{k=0}^\infty \exp\left(-\gamma^k \log\left(\frac{1}{\delta_\Delta}\right)\right) \\ 
&= 2^{2- \beta}\left(\sum_{k=0}^\infty \delta_{\Delta}^{\gamma^k} + \sum_{k=0}^{k^*-1} \delta_{\gamma^{k+1}} \right) \\
&= 2^{2- \beta}\left(\delta_{\Delta} \sum_{k=0}^\infty 2^{1 - \gamma^k} + \sum_{k=0}^{k^*-1} \delta_{\gamma^{k+1}} \right) \\
&\leq c_\gamma 2^{-\beta} \left(\delta_{\Delta} + \sum_{k=0}^{k^* - 1} \delta_{\gamma^{k+1}} \right) \\
&= 2^{-\beta} \tilde \delta_{\Delta} \,.
}

\section{Proof of Regularity Lemmas}\label{app:uniform}

I make use of the following version of Chernoff's bound.

\begin{lemma}[Chernoff Bound]\label{lem:chernoff}
Let $X_1,\ldots,X_n$ be independent Bernoulli random variables with $\E X_i \leq \mu$. Then
\eq{
\P{\frac{1}{n} \sum_{t=1}^n X_t \geq \mu + \epsilon} \leq \exp\left(-\frac{n\epsilon^2}{3\mu}\right)\,.
}
\end{lemma}

\begin{proof}[Lemma \ref{lem:uniform1}]
For $k \in \set{0,1,\ldots}$ let $S_k = \set{2^k,\ldots,\min\set{K, 2^{k+1}-1}}$.
Define $k_{\max} = \min \set{k : K \in S_k}$, which means for $k < k_{\max}$ we have $|S_k| = 2^k$ and 
$\bigcup_{k=0}^{k_{\max}} S_k = \set{1,\ldots,K}$.
Define $S_{k,\beta=1} = \set{i \in S_k : \beta_{i,\Delta} = 1}$. Then by Chernoff's bound (Lemma \ref{lem:chernoff}) and the union bound
\eq{
\P{\exists k \in \set{0,1,\ldots,k_{\max} - 1} : |S_{k,\beta=1}| \geq 2^{k-1}} 
\leq \sum_{k < k_{\max}} \exp\left(-\frac{2^{k - 2}}{3 \tilde \delta_\Delta }\right) 
\leq \sum_{k=0}^\infty \frac{3 \tilde \delta_\Delta}{2^{k-2}} 
= 24 \tilde \delta_\Delta \,. 
}
Now we assume $|S_{k,\beta=1}| \leq 2^{k-1}$ for all $k \in \set{0,1,\ldots,k_{\max}}$, then
\eq{
\sum_{i : \beta_{i,\Delta} = 1} \min\set{u_i, T} 
&\sr{(a)}\geq \sum_{k < k_{\max}} \sum_{i \in S_{k,\beta=1}} \min\set{u_i, T} 
\sr{(b)}\geq \sum_{k < k_{\max}} |S_{k,\beta=1}| \min\set{u_{2^{k+1}}, T} \\
&\sr{(c)}\geq \sum_{k < k_{\max}} \frac{|S_k|}{2} \min\set{u_{2^{k+1}}, T} 
\sr{(d)}\geq \sum_{k < k_{\max}} \frac{|S_{k+1}|}{4} \min\set{u_{2^{k+1}}, T} \\
&\sr{(e)}\geq \sum_{k < k_{\max}} \frac{1}{4} \sum_{i \in S_{k+1}} \min\set{u_i, T} 
\sr{(f)}\geq \frac{1}{4} \sum_{k=2}^K \min\set{u_i, T}\,,
}
where (a) is trivial,
(b) since $u_i$ is non-increasing,
(c) since $|S_k| = 2^k$ and $|S_{k,\beta=1}| \geq 2^{k-1}$,
(d) since $|S_{k+1}| \leq 2^{k+1}$,
(e) since $u_i$ is non-increasing,
(f) since $\bigcup_{k < k_{\max}} S_{k+1} = \set{2,\ldots,K}$.
Finally note that $\beta_{1,\Delta} = 1$, since $|S_{0,\beta>1}| \leq 1/2$. 
Therefore 
$5 \sum_{i : \beta_{i,\Delta} = 1} \min\set{u_i, T} 
\geq \sum_{k=1}^K \min\set{u_i, T}$.
\end{proof}

\begin{proof}[Lemma \ref{lem:uniform2}]
We make a similar argument as above. Let $S_k$ and $k_{\max}$ be as in the proof of Lemma \ref{lem:uniform1} and 
$S_{k,\beta} = \set{i \in S_k : \beta_{i,\Delta} \geq \beta}$.
Then by Lemma \ref{lem:conc} and Chernoff's bound we have
\eq{
&\P{\exists k \in \set{0,1,\ldots,k_{\max}} \text{ and } \beta \in \set{2,3,\ldots} : |S_{k,\beta}| \geq |S_k| 2^{-\beta/4}} \\ 
&\leq \sum_{\beta=2}^\infty 3\tilde\delta_\Delta 2^{-\beta/2 - 1} + \sum_{k=0}^\infty \sum_{\beta=2}^\infty 3\tilde \delta_\Delta 2^{-k} 2^{-\beta/2} 
\leq 13 \tilde \delta_{\Delta}\,.
}
Now assume that $|S_{k,\beta}| \leq |S_k| 2^{-\beta/4}$ for all $k \in \set{0,1,\ldots,k_{\max}}$ and $\beta \in \set{2,3,\ldots}$. Then
\eq{
\sum_{i=1}^K \beta_i \min\set{u_i, T} 
&\sr{(a)}= \sum_{k=0}^{k_{\max}} \sum_{i \in S_k} \beta_i \min\set{u_i, T} \\
&\sr{(b)}\leq \sum_{i=1}^K \min\set{u_i, T} + \sum_{k=1}^{k_{\max}} \sum_{\beta=2}^\infty \sum_{i \in S_{k,\beta}} \beta \min\set{u_i, T} \\
&\sr{(c)}\leq \sum_{i=1}^K \min\set{u_i, T} + \sum_{k=1}^{k_{\max}} \sum_{\beta=2}^\infty |S_k| 2^{- \beta/4} \beta \min\set{u_{2^k}, T} \\
&\sr{(d)}\leq \sum_{i=1}^K \min\set{u_i, T} + \sum_{k=1}^{k_{\max}} 33 \cdot |S_k| \min\set{u_{2^k}, T} \\
&\sr{(e)}\leq \sum_{i=1}^K \min\set{u_i, T} + \sum_{k=1}^{k_{\max}} 66 \cdot |S_{k-1}| \min\set{u_{2^k}, T} \\
&\sr{(f)}\leq \sum_{i=1}^K \min\set{u_i, T} + \sum_{k=1}^{k_{\max}} 66 \sum_{i \in S_{k-1}} \min\set{u_i, T} \\
&\sr{(g)}\leq 67\sum_{i=1}^K \min\set{u_i, T}\,,
}
where (a) is trivial,
(b) from the definition of $S_{k,\beta}$,
(c) since we have assumed that $|S_{k,\beta}| \leq |S_k| 2^{-\beta/4}$ and by the monotonicity of $u_i$,
(d) by evaluating the (almost) geometric series,
(e) since $|S_k| \leq |S_{k-1}|$ for all $k \geq 1$,
(f) since $u_i$ is non-increasing and
(g) is trivial.
\end{proof}

\section{Technical Lemmas}\label{app:tech}

\begin{lemma}\label{lem:delta4}
If $\Delta \geq \Delta_i$, then 
$\displaystyle \frac{\delta_{\Delta_i}}{\delta_{\Delta}} \leq \frac{\Delta^2}{\Delta_i^2}$.
\end{lemma}

\begin{proof}
The result follows from the definition of $\delta_{\Delta_i}$ and the fact that
\eq{
\sum_{j=1}^K \min\set{u_i, u_j} 
&= \sum_{j=1}^K \min\set{\frac{c_{10}}{\Delta_i^2} \log\left(\frac{c_9}{\delta_{u_i}}\right), u_j}
\leq \frac{\Delta^2}{\Delta_i^2}\cdot \sum_{j=1}^K \min\set{\frac{c_{10}}{\Delta} \log\left(\frac{c_9}{\delta_{u_i}}\right), u_j} \\
&\leq \frac{\Delta^2}{\Delta_i^2}\cdot \sum_{j=1}^K \min\set{u_{\Delta}, u_j}\,,
}
where the last inequality follows since $u_i \geq u_\Delta$ and so $\delta_{u_i} \geq \delta_{u_\Delta}$.
\end{proof}

\begin{lemma}\label{lem:u}
If $\Delta \leq \Delta_i$, then $\displaystyle \frac{u_\Delta}{u_i} \leq \frac{\Delta_i^2}{\Delta^2}$.
\end{lemma}

\begin{proof}
Since $\Delta \leq \Delta_i$ we have $u_\Delta \geq u_i$ and so $\delta_{\Delta} \geq \delta_{u_i}$. Therefore
\eq{
\frac{u_\Delta}{u_i} = \frac{\Delta_i^2 \log \left(\frac{c_9}{\delta_\Delta}\right)}{\Delta^2 \log\left(\frac{c_9}{\delta_{u_i}}\right)}
\leq \frac{\Delta_i^2}{\Delta^2}
}
as required.
\end{proof}

\newpage
\section{Constants and Constraints}\label{app:constants}

Here we analyse the various constants and corresponding constraints used in the proof of \cref{thm:prob-dep}.
We have the following constraints.
\newcounter{oldequation}
\setcounter{oldequation}{\value{equation}}
\renewcommand\theequation{Const\arabic{equation}}
\setcounter{equation}{0}

\noindent
\begin{minipage}[t]{6.2cm}
\eqn{
\label{C:c6}        &c_6 = 134 c_2 c_8^3 / \psi \\
\label{C:c7}        &c_7 = 10\psi c_6(c_1 + c_2) \\
\label{C:c_gamma}   &c_\gamma = \sum_{k=0}^\infty 2^{3 - \gamma^k} \\
\label{C:c11}       &c_4 = \sqrt{2}\gamma \\
\label{C:c12}       &c_5 = \frac{24\gamma c_6 c_\gamma}{n(\gamma - 1)} \\ 
\label{C:psi}       &\psi \geq 2 \\
\label{C:c3}        &(2+c_3)c_4 / c_8 \leq 1/4  \\
\label{C:gamma}     &\gamma \in (1, \alpha / 2)\,.
}
\end{minipage}
\hfill
\begin{minipage}[t]{8.6cm}
\eqn{
\label{C:c2}  \sqrt{\alpha} - \sqrt{\frac{\alpha}{c_2}} - \sqrt{2\gamma} - \sqrt{\frac{2\gamma}{c_2}} &\geq 1 \\
\label{C:c1} \frac{1}{2} \sqrt{\frac{c_1}{u_\Delta} \log\left(\frac{1}{\delta_{\Delta}}\right)} &> 2c_4 \Delta \\
\label{C:c1-2} \sqrt{\frac{2\gamma c_1}{u_\Delta} \log\left(\frac{c_1}{\delta_{\Delta}}\right)} &\leq c_3 c_4 \Delta \\
\label{C:uj-1} \sqrt{\frac{2\gamma c_2}{u_\Delta} \log\left(\frac{c_2}{\delta_{\Delta}}\right)} &\leq \Delta / 4 \\
\label{C:uj-2} \sqrt{\frac{2\gamma}{u_\Delta} \log\left(\frac{1}{\delta_{\Delta}}\right)} &\leq \Delta / 4 \\
\label{C:uj-3} \sqrt{\frac{\alpha}{u_\Delta} \log\left(\frac{c_7}{\delta_{\Delta}}\right)} &\leq \Delta / 4\,. 
}
\end{minipage}
\eqn{
\label{C:c13}  c_{11} &= 2c_5 c_8(1 + \log(c_8)) + c_5 (3c_8 \log(c_8) + 5) \,.
}
\subsection*{Satisfying the Constraints}
First we satisfy \CONref{C:c2} by choosing
\eq{
c_2 = \left(\frac{\sqrt{\alpha} + \sqrt{2\gamma}}{\sqrt{\alpha}-\sqrt{2\gamma}}\right)^2\,,
}
which by the assumption that $\gamma < \alpha / 2$ is finite.
We observe that Eq. (\ref{C:c1-2}--\ref{C:uj-3}) are satisfied by choosing
\eq{
u_\Delta \geq 
\max\set{
  \frac{16\cdot 2\gamma c_2}{\Delta^2} \log\left(\frac{c_2}{\delta_{u_\Delta}}\right),
  \frac{16 \alpha}{\Delta^2} \log\left(\frac{c_7}{\delta_{u_\Delta}}\right),
  \frac{2\gamma c_1}{c_3^2 c_4^2 \Delta^2} \log\left(\frac{c_1 \psi}{\delta_{\Delta}}\right)
}
}
For the sake of simplicity we will be conservative by choosing
\eq{
u_\Delta &= \frac{c_9}{\Delta^2} \log\left(\frac{c_{10}}{\delta_\Delta}\right) &
c_9 &= \max\set{32 \gamma c_2,\, 16\alpha,\, \frac{2\gamma c_1}{c_3^2 c_4^2}} &
c_{10} &= \max\set{c_2,\, c_7 \psi,\, c_1 \psi}\,.
}
Then \CONref{C:c1} can be satisfied by choosing $c_1$ and $c_3$ sufficiently large. 
So increasing $c_1$ increases $u_\Delta$, but the latter dependence is logarithmic, which means that for sufficiently large $c_1$ the relation will be
satisfied. Now \CONref{C:c3} can be satisfied by choosing $c_8$ sufficiently large. 
Now \CONref{C:gamma} can be satisfied provided $\alpha > 2$, which we assumed in both \cref{thm:prob-dep} and \cref{thm:prob-ind}.

\setcounter{equation}{\value{oldequation}}
\renewcommand\theequation{\arabic{equation}}

\section{Computation Time}\label{app:comp}

A naive implementation of the \textit{Optimally Confident UCB} algorithm requires $O(K)$ computation per time step.
For large $K$ it is possible
to obtain a significant performance gain by noting that the index of unplayed arms is strictly decreasing,
which means the algorithm only needs to re-sort the arms for which the index at the time of last play exceeds the index
of the previously played arm. The running time of this algorithm over $n$ time steps is $O(n)$ in 
expectation (asymptotically).
This observation also applies to MOSS, for which the index of unplayed arms does not change with time at all.\footnote{This property
makes it trivial to implement MOSS in $O(\log K)$ per time step using a priority queue, but for long horizons one should expect
even better performance.}
In contrast, for Thompson sampling it seems that sampling all arms at every time step is essentially unavoidable without
significantly changing the algorithm. 

\section{Lower Bounds}\label{sec:lower}

Throughout this section I consider a single fixed and arbitrary policy $\pi$.
Starting with a simple case, let $\Delta > 0$ and define $\mu^i \in \R^K$ for $i \in \set{1,\ldots,K}$ by
\eq{
\mu^i_k = \begin{cases}
\Delta & \text{if } k = 1 \\
2\Delta & \text{if } k = i \\
0 & \text{otherwise}\,.
\end{cases}
}
Let $\E_i$ denote the expectation with respect to the measure on outcomes induced by
the combination of the fixed strategy with environment $\mu^i$ and standard Gaussian noise. Let $\Ps_i$ be the corresponding measure.

\begin{theorem}\label{thm:lower}
Assume $H = (K-1)/\Delta^2 \leq n/e$. 
Then there exists an $i$ such that
\eq{
R^\pi_{\mu^i}(n) \geq \frac{1}{4} \cdot \frac{K-1}{\Delta} \log\left(\frac{n}{H}\right)\,.
}
\end{theorem}

\begin{remark}
Up to constant factors $H$ coincides with $H_j$ for all $j$ and all reward vectors $\mu^i$ so the theorem implies
the upper bound in \cref{thm:prob-dep} is tight for at least one of the reward vectors $\mu^i$.
\end{remark}

\begin{proof}[\cref{thm:lower}]
Define $A_i = \ind{T_i(n+1) \geq n/2}$ be the event that the $i$th arm is chosen at least $n/2$ times.
Suppose that
\eqn{
\label{eq:lower1}
(\exists i > 1) \qquad
\E_1 T_i(n+1) \leq \frac{1}{2\Delta^2} \log\left(\frac{n}{H \log\frac{n}{H}}\right)\,.
}
Then an application of Lemma 2.6 by \cite{Tsy08} leads to 
\eq{
\Pn{1}{A_i} + \Pn{i}{\neg A_i} 
\geq \exp\left(-\KL(\Ps_1,\Ps_i)\right) 
= \exp\left(-2\Delta^2\E_1 T_i(n+1)\right)
\geq \frac{H}{n} \log \left(\frac{n}{H}\right)\,.
}
Therefore 
\eq{
R^\pi_{\mu^1}(n) + R^\pi_{\mu^i}(n) 
\geq \frac{n\Delta}{2} \cdot \frac{H}{n} \log\left(\frac{n}{H}\right)
= \frac{1}{2} \cdot \frac{K-1}{\Delta} \log\left(\frac{n}{H}\right)\,.
}
If \cref{eq:lower1} does not hold, then
\eq{
R^\pi_{\mu^1}(n) \geq \frac{(K-1)}{2\Delta} \log\left(\frac{n}{H \log \frac{n}{H}}\right)
\geq \frac{1}{4} \cdot \frac{K-1}{\Delta} \log\left(\frac{n}{H}\right)\,,
}
where the last inequality follows since $\log(x/\log(x)) \geq \log(x)/2$ for $x \geq e$.
Therefore we conclude that there exists an $i$ such that
\eq{
R^\pi_{\mu^i}(n) \geq \frac{1}{4} \cdot \frac{K-1}{\Delta} \log\left(\frac{n}{H}\right)
}
as required.
\end{proof}

The lower bound  
matches the upper bound for this problem given in \cref{thm:prob-dep}.
It should be emphasised that if \cref{eq:lower1} does not hold by a largish margin, then
the penalty in environment $i$ is enormous relative to the logarithmic penalty of exploring arm $i$,
which means that in some sense it is optimal to explore arm $i$ such that
\eq{
\E_1[T_i(n+1)] \in \Omega\left(\frac{1}{\Delta^2} \log \left(\frac{n}{H}\right)\right)\,.
}

\subsubsect{Unbalanced reward vector}
For the case that $\mu$ is arbitrary and $\mu^i_j = \mu_j + 2 \Delta_j \ind{i = j}$ it is possible to show that there exists an $i$ such that
\eq{
R^\pi_{\mu^i}(n) \in \Omega\left( \sum_{j \neq i} \frac{1}{\Delta_j + \Delta_i} \log\left(\frac{1}{\log K} \cdot \frac{n}{\sum_{k \neq i} 
\min\set{\left(\frac{1}{\Delta_k + \Delta_i}\right)^2,\, \left(\frac{1}{\Delta_j + \Delta_i}\right)^2}}\right)\right)\,.
}
This matches the upper bound given in \cref{thm:prob-dep} except for the extraneous $\log K$ in the denominator of the logarithm. I believe the upper bound is
tight, which is corroborated in certain cases including the uniform case explored in \cref{thm:lower} and the highly non-uniform case discussed in \cref{sec:conjecture}.
The omitted proof of the above result is an algebraic mess, but follows along the same lines as \cref{thm:lower}.

\section{Almost Optimally Confident UCB}\label{app:almost}

Here I present a practical and less aggressive version of \cref{alg:ocucb} that manages the same regret as improved UCB by \cite{AO10}.
While the regret guarantees are not quite optimal, the proof is so straightforward it would be remiss not to include it.

\begin{center}
\begin{minipage}{9cm}
\begin{algorithm}[H]
\KwIn{$K$, $n$}
Choose each arm once \\
\For{$t \in K+1,\ldots,n$} {
Choose $\displaystyle I_t = \argmax_i \hat \mu_i(t) + \sqrt{\frac{2}{T_i(t)} \log\left(\frac{n}{T_i(t)}\right)}$
}
\caption{Almost Optimally Confident UCB}\label{alg:aocucb}
\end{algorithm}
\end{minipage}
\end{center}

\begin{theorem}\label{thm:aocucb}
There exist universal constants $C_3$ and $C_4$ such that for all $\delta \geq 0$, 
\eq{
R_\mu^{\text{aocucb}}(n) \leq n\delta + \sum_{i: \Delta_i > \delta} \frac{C_3}{\Delta_i} \logp\left(n\Delta_i^2\right)
\quad \text{ and }\quad
R_\mu^{\text{aocucb}}(n) \leq C_4 \sqrt{nK \log K}\,.
}
\end{theorem}
Some remarks before the proof:
\begin{itemize}
\item The constant appearing inside the square root is the smallest known for a UCB-style algorithm with finite-time guarantees. Other algorithms require
at least $2 + \epsilon$ with arbitrary $\epsilon > 0$, but with a bound that tends to infinity as $\epsilon$ becomes small. There are asymptotic results
when the constant is $2$ by \cite{KR95}, which leaves open the possibility for improved analysis.
\item The algorithm is strictly less aggressive than both MOSS and OCUCB, which eases the analysis and saves it from the poor problem
dependent regret of MOSS.
\end{itemize}

\newcommand{\leqtimes}{\stackrel{\times}\leq}

\begin{proof}[\cref{thm:aocucb}]
I write $f(\cdot) \leqtimes g(\cdot)$ if there is a universal constant $c$ such that $f(\cdot) \leq c \cdot g(\cdot)$.
First I note that for any $\Delta > 0$
\eq{
\P{\exists t \leq n : \hat \mu_1(t) + \sqrt{\frac{2}{T_1(t)} \log\left(\frac{n}{T_1(t)}\right)} \leq \mu_1 - \Delta} \leqtimes \frac{1}{n\Delta^2} \log \left(n\Delta^2\right)\,.
}
The proof of this claim follows from a peeling device on a geometric grid with parameter $\gamma$ that must then be optimised.
Let 
\eq{
\Delta = \min \set{\mu_1 - \hat \mu_1(t) - \sqrt{\frac{2}{T_1(t)} \log \left(\frac{n}{T_1(t)}\right)} : t \leq n}\,.
}
For each sub-optimal arm $i$ define a stopping time
\eq{
\tau_i = \min\set{t : \hat \mu_i(t) + \sqrt{\frac{2}{T_i(t)} \log\left(\frac{n}{T_i(t)}\right)} \leq \mu_i + \Delta_i / 2}\,.
}
This is essentially identical to that used by \cite{AB09} where it is shown that
\eq{
\E[T_i(\tau_i)] \leqtimes \frac{1}{\Delta_i^2} \logp( n\Delta_i^2)\,.
}
Now if $\Delta \leq \Delta_i / 2$, then $T_i(n+1) \leq T_i(\tau_i)$. 
Therefore
\eq{
\E T_i(n+1) \leq \E\left[\tau_i \ind{\Delta \leq \frac{\Delta_i}{2}} + n \ind{\Delta > \frac{\Delta_i}{2}}\right]
\leqtimes \frac{1}{\Delta_i^2} \logp(n \Delta_i^2)\,.
}
The result follows by bounding $\sum_{i : \Delta_i \leq \delta} T_i(n+1) \leq n$.
\end{proof}

\section{Experiments}\label{app:exp}

Before the experiments some book-keeping.
All code will be made available in any final version.
Error bars depict two standard errors and are omitted when they are too small to see.
Each data-point is an i.i.d.\ estimate based on $N$ samples, which is given in the title of all plots. The noise model is a standard Gaussian in all
experiments ($\eta_t \sim \mathcal N(0, 1)$).
I compare the new algorithm with a variety of algorithms in different regimes. First though I evaluate the sensitivity of OCUCB to 
the main parameter $\alpha$ in two key regimes. The first is when the horizon is fixed to $n=10^4$ and there is a single
optimal arm and $\Delta_i = \Delta$ for all suboptimal arms. The second regime is like the first, but $\Delta = 3/10$ is fixed
and $n$ is varied. The results (see \cref{fig:sense}) unsurprisingly show that the optimal $\alpha$ is problem specific, but that $\alpha \in [2, 3]$ is a reasonable
choice in all regimes. In general, a large $\alpha$ leads to better performance when $n$ is small while small $\alpha$ is better when $\Delta$ is large.
This is consistent with the intuition that large $\alpha$ makes the algorithm more conservative. The dependence on $\psi$ is very weak (results are omitted).

\begin{figure}[H]
\centering
\begin{tikzpicture}[font=\scriptsize]
\begin{groupplot}[group style={rows=2,columns=2,horizontal sep=30pt,vertical sep=50pt},
  xtick=data,scaled x ticks=false,width=7.7cm,height=4.5cm,xlabel near ticks,ylabel near ticks,compat=newest]

\nextgroupplot[title={$n = 10^4$ and $K = 2$ and $N = 2.3\times10^5$ and $\Delta$ varies},xmin=0,xmax=2,
  xtick={0,2},
  xlabel={$\Delta$},
  ylabel={Expected regret}]
    \addplot+[] table[x index=0,y index=1] \dataFour;
    \addlegendentry{$\alpha = 1$};
    \addplot+[] table[x index=0,y index=2] \dataFour;
    \addlegendentry{$\alpha = 2$};
    \addplot+[] table[x index=0,y index=3] \dataFour;
    \addlegendentry{$\alpha = 3$};
    \addplot+[] table[x index=0,y index=4] \dataFour;
    \addlegendentry{$\alpha = 6$};

\nextgroupplot[title={$n=10^4$ and $K = 10$ and $N = 1.2\times 10^5$ and $\Delta$ varies},xtick={0,2},xlabel={$\Delta$},xmin=0,xmax=2]
    \addplot+[] table[x index=0,y index=1] \dataFive;
    \addplot+[] table[x index=0,y index=2] \dataFive;
    \addplot+[] table[x index=0,y index=3] \dataFive;
    \addplot+[] table[x index=0,y index=4] \dataFive;

\nextgroupplot[title={$\Delta=2/10$ and $K = 2$ and $N = 5\times10^5$ and $n$ varies},xtick={5000,100000},xlabel={$n$},ylabel={Expected regret},ymax=150,xmin=5000,xmax=100000]
    \addplot+[] table[x index=0,y index=1] \dataSix;
    \addplot+[] table[x index=0,y index=3] \dataSix;
    \addplot+[] table[x index=0,y index=5] \dataSix;
    \addplot+[] table[x index=0,y index=7] \dataSix;

\nextgroupplot[title={$\Delta=2/10$ and $K = 10$ and $N = 2.8\times 10^5$ and $n$ varies},xtick={5000,100000},xlabel={$n$},ymax=800,xmin=5000,xmax=100000]
    \addplot+[] table[x index=0,y index=1] \dataSeven;
    \addplot+[] table[x index=0,y index=3] \dataSeven;
    \addplot+[] table[x index=0,y index=5] \dataSeven;
    \addplot+[] table[x index=0,y index=7] \dataSeven;
\end{groupplot}
\end{tikzpicture}
\caption{Parameter sensitivity}\label{fig:sense}
\end{figure}
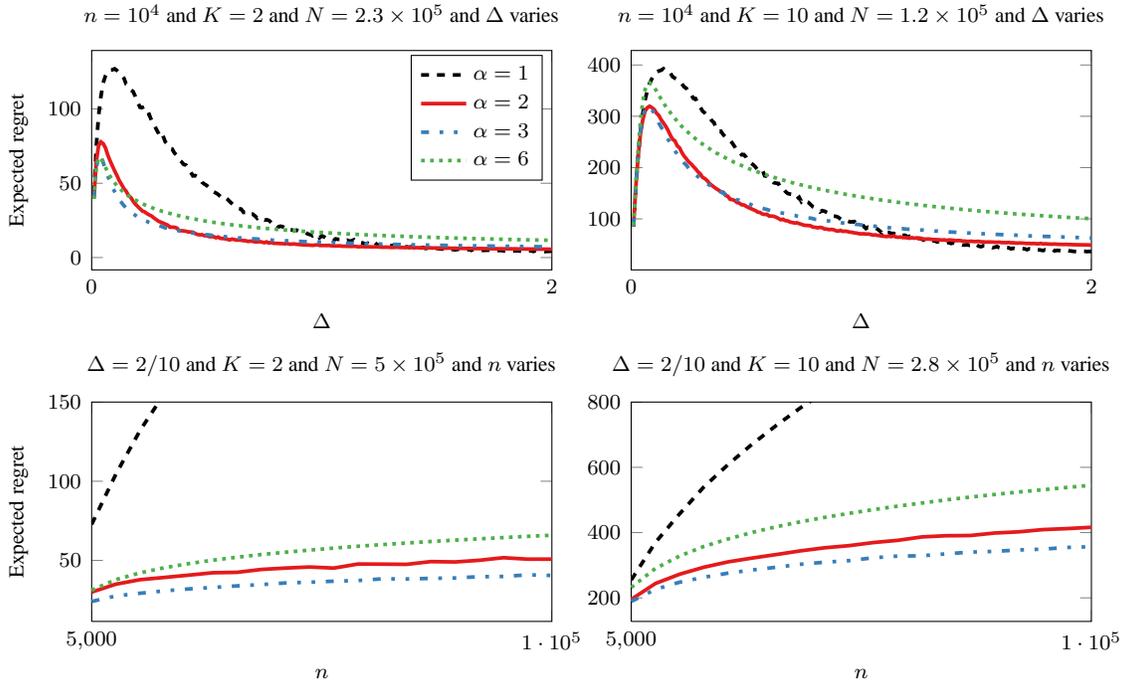

\subsection*{Comparison to Other Algorithms}

I compare OCUCB and AOCUCB against UCB, MOSS, Thompson Sampling with a flat Gaussian prior and the near-Bayesian finite-horizon Gittins index strategy.
For OCUCB I used $\alpha = 3$ in all experiments, which was chosen based on the experiments in the previous section.
For UCB and MOSS I used the following indexes:
\eq{
I_t^{\text{ucb}} &= \argmax_i \hat \mu_i(t) + \sqrt{\frac{2}{T_i(t)} \log t} 
& I_t^{\text{moss}} &= \argmax_i \hat \mu_i(t) + \sqrt{\frac{2}{T_i(t)} \log \max\set{1, \frac{n}{T_i(t) K}}}\,.
}
This version of UCB is asymptotically optimal \citep{KR95}, but finite-time results are unknown as far as I am aware. In practice the $2$ inside the square root is uniformly
better than any larger constant. The analysis of \cite{AB09} would suggest using a constant of $4$ for MOSS, but $2$ led to improved empirical results
and in fact the theoretical argument can be improved to allow $2 + \epsilon$ for any $\epsilon > 0$ using the same arguments as this paper. 
The finite-horizon Gittins index strategy is a near-Bayesian strategy
introduced for one-armed bandits by \cite{BJK56} and suggested as an approximation of the Bayesian strategy in the general case by \cite{Nin11} and possibly others. 
For Bernoulli noise it was shown to have excellent
empirical performance by \cite{KCOG12} while theoretical and empirical results are recently given for the Gaussian case by \cite{Lat15gittins}.
The Gittins index strategy is not practical computationally for horizons larger than $n = 10^4$, so
is omitted from the large-horizon plots.
For Thompson sampling I used the flat Gaussian prior, which means that $I_t = t$ for $t \in \set{1,\ldots,K}$ and thereafter
\eq{
I_t^{\text{thomp.\ samp.}} = \argmax_i \hat \mu_i(t) + \eta_i(t)\,,
}
where $\eta_i(t) \sim \mathcal N(0, 1/T_i(t))$.
In the first set of experiments I use the same regimes as \cref{fig:sense}. 

\begin{figure}[H]
\centering
\begin{tikzpicture}[font=\scriptsize]
\begin{groupplot}[group style={rows=2,columns=2,horizontal sep=20pt,vertical sep=50pt},
  xtick=data,scaled x ticks=false,width=7.9cm,height=5.5cm,xlabel near ticks,ylabel near ticks]

\nextgroupplot[title={$n = 10^4$ and $K = 2$ and $N = 1.2 \times 10^5$ and $\Delta$ varies},
  xtick={0,1},
  xlabel={$\Delta$},
  xmin=0,
  xmax=1,
  legend cell align=left,
  ylabel={Expected regret}]
    \addplot+[smooth] table[x index=0,y index=1] \dataOne;
    \addlegendentry{UCB};
    \addplot+[smooth] table[x index=0,y index=2] \dataOne;
    \addlegendentry{OCUCB};
    \addplot+[smooth] table[x index=0,y index=3] \dataOne;
    \addlegendentry{AOCUCB};
    \addplot+[smooth] table[x index=0,y index=4] \dataOne;
    \addlegendentry{MOSS};
    \addplot+[smooth] table[x index=0,y index=7] \dataOne;
    \addlegendentry{Thompson sampling};
    \addplot+[smooth] table[x index=0,y index=6] \dataOne;
    \addlegendentry{Gittins};

\nextgroupplot[title={$n = 10^4$ and $K = 10$ and $N = 3\times 10^4$ and $\Delta$ varies},
  xtick={0,1},
  xlabel={$\Delta$},
  xmin=0,
  xmax=1]
    \addplot+[smooth] table[x index=0,y index=1] \dataTwo;
    \addplot+[smooth] table[x index=0,y index=2] \dataTwo;
    \addplot+[smooth] table[x index=0,y index=3] \dataTwo;
    \addplot+[smooth] table[x index=0,y index=4] \dataTwo;
    \addplot+[smooth] table[x index=0,y index=7] \dataTwo;
    \addplot+[smooth] table[x index=0,y index=6] \dataTwo;

\nextgroupplot[title={$\Delta=3/10$ and $K = 2$ and $N = 2\times 10^6$ and $n$ varies},xtick={0,100000},xlabel={$n$},ylabel={Expected regret},xmin=0,xmax=100000]
    \addplot+[] table[x index=0,y index=1] \dataEight;
    \addplot+[] table[x index=0,y index=2] \dataEight;
    \addplot+[] table[x index=0,y index=3] \dataEight;
    \addplot+[] table[x index=0,y index=4] \dataEight;
    \addplot+[] table[x index=0,y index=6] \dataEight;

\nextgroupplot[title={$\Delta=3/10$ and $K = 10$ and $N = 4.8\times 10^4$ and $n$ varies},xtick={0,100000},xlabel={$n$},xmin=0,xmax=100000]
    \addplot+[] table[x index=0,y index=1] \dataNine;
    \addplot+[] table[x index=0,y index=2] \dataNine;
    \addplot+[] table[x index=0,y index=3] \dataNine;
    \addplot+[] table[x index=0,y index=4] \dataNine;
    \addplot+[] table[x index=0,y index=6] \dataNine;

\end{groupplot}
\end{tikzpicture}
\caption{Regret comparison}\label{fig:regret}
\end{figure}
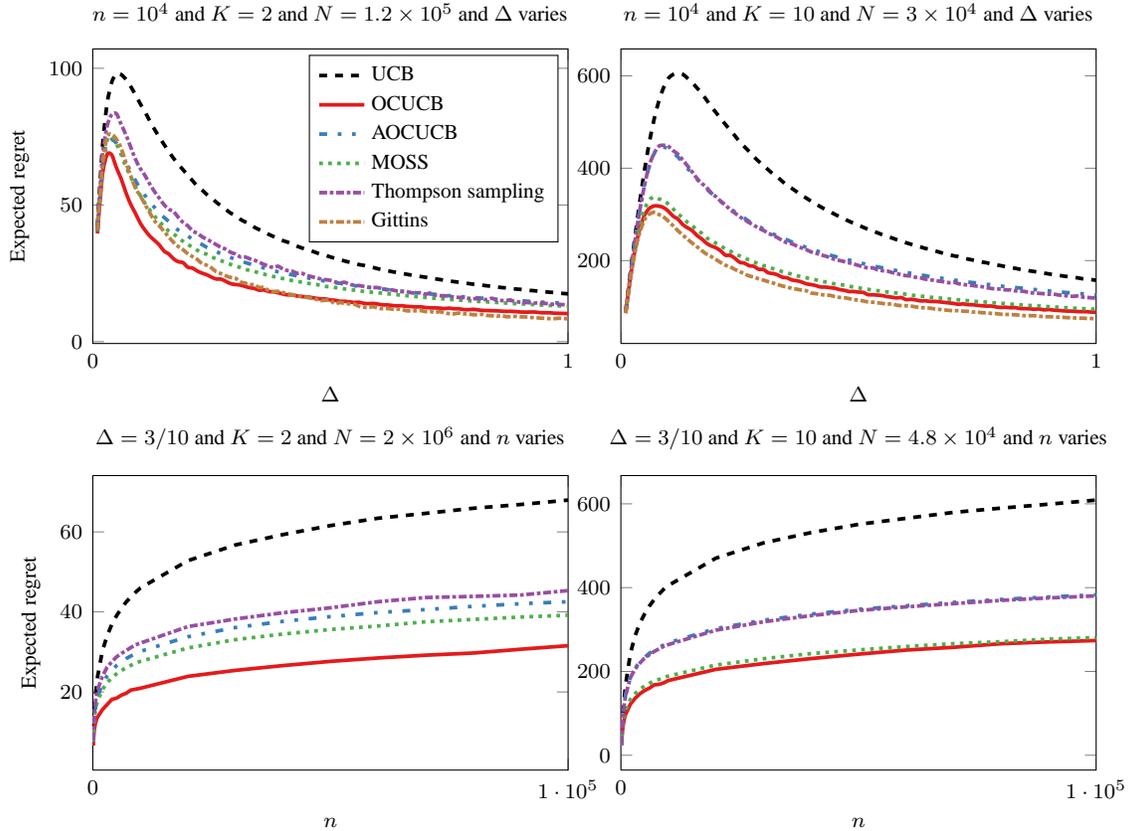

The results show that OCUCB is always competitive with the best and sometimes significantly better. Arguably the Gittins strategy is
the winner for small horizons, but its computation is impractical for large horizons. 
MOSS is also competitive in these regimes, which is consistent with the theory (in the next section we see where things go wrong for MOSS).
Thompson sampling and AOCUCB are almost indistinguishable.

\subsection*{Failure of MOSS}

The following experiment highlights the poor problem-dependent performance of MOSS relative to OCUCB. The experiment
uses 
\eq{
\mu_1 &= 0 &
\mu_2 &= -\frac{1}{4K} &
\mu_i &=-1 \text{ for all }i > 2 &
n &= K^3\,.
}
The results are plotted for increasing $K$ and algorithms OCUCB/MOSS. 
Algorithms like Thompson sampling for which the running time is $O(Kn)$ are too slow to evaluate in this regime for large $K$. 
As the theory predicts, the regret of MOSS is exploding for large $K$, while OCUCB enjoys good performance. Curiously the issues
are only serious when $K$ (and so $n$) is unreasonably large. In modestly sized experiments MOSS is usually only slightly worse
than OCUCB.

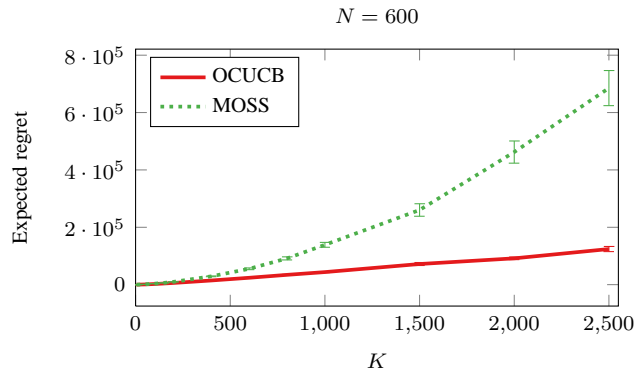
\begin{figure}[H]
\centering
\begin{tikzpicture}[font=\scriptsize]
\begin{groupplot}[group style={rows=1,columns=1},width=8cm,height=5cm,xlabel={$K$},ylabel={Expected regret},scaled y ticks=false,ylabel shift=5pt,compat=newest,legend pos=north west,legend cell align=left]
\nextgroupplot[title={$N = 600$},xmin=0,xmax=2550]
  \addplot+[C2,solid,error bars/.cd,y dir=both,y explicit,error bar style={solid}] table[x index=0,y index=1,y error index=3] \dataTen;
  \addlegendentry{OCUCB}
  \addplot+[C4,dotted,error bars/.cd,y dir=both,y explicit,error bar style={solid}] table[x index=0,y index=2,y error index=4] \dataTen;
  \addlegendentry{MOSS}
\end{groupplot}
\end{tikzpicture}
\caption{Failure of MOSS}\label{fig:mossfailure}
\end{figure}

\subsection*{Uniformly Distributed Arms}

In the final experiment I set $\mu_i = -(i-1) / K$ for all $i$ and vary $n$ with $K \in \set{10, 100}$.
As in previous experiments we see OCUCB and MOSS leading the pack with Thompson sampling and AOCUCB almost identical
and UCB significantly worse. 

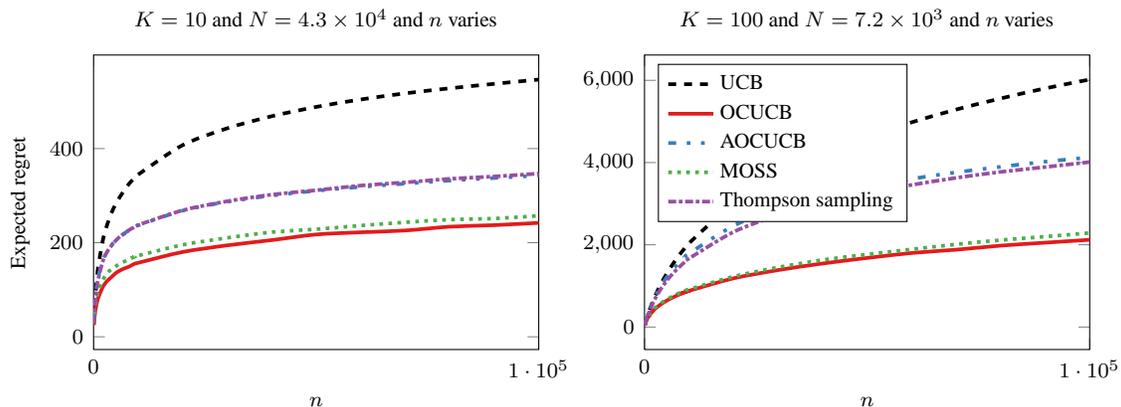
\begin{figure}[H]
\centering
\begin{tikzpicture}[font=\scriptsize]
\begin{groupplot}[group style={rows=1,columns=2,horizontal sep=40pt,vertical sep=50pt},
  xtick=data,scaled x ticks=false,width=7.5cm,height=5.5cm,xlabel near ticks,ylabel near ticks]

\nextgroupplot[title={$K = 10$ and $N = 4.3\times 10^4$ and $n$ varies},
  xtick={0,100000},
  xlabel={$n$},
  xmin=0,
  xmax=100000,
  ylabel={Expected regret}]
    \addplot+[smooth] table[x index=0,y index=1] \dataEleven;
    \addplot+[smooth] table[x index=0,y index=2] \dataEleven;
    \addplot+[smooth] table[x index=0,y index=3] \dataEleven;
    \addplot+[smooth] table[x index=0,y index=4] \dataEleven;
    \addplot+[smooth] table[x index=0,y index=5] \dataEleven;
\nextgroupplot[title={$K = 100$ and $N = 7.2\times 10^3$ and $n$ varies},
  xtick={0,100000},
  xlabel={$n$},
  xmin=0,
  legend cell align=left,
  legend pos=north west,
  xmax=100000]
    \addplot+[smooth] table[x index=0,y index=1] \dataTwelve;
    \addlegendentry{UCB};
    \addplot+[smooth] table[x index=0,y index=2] \dataTwelve;
    \addlegendentry{OCUCB};
    \addplot+[smooth] table[x index=0,y index=3] \dataTwelve;
    \addlegendentry{AOCUCB};
    \addplot+[smooth] table[x index=0,y index=4] \dataTwelve;
    \addlegendentry{MOSS};
    \addplot+[smooth] table[x index=0,y index=5] \dataTwelve;
    \addlegendentry{Thompson sampling};
\end{groupplot}
\end{tikzpicture}
\caption{Uniformly distributed arms}
\end{figure}

\section{Table of Notation}\label{app:notation}

\noindent
\renewcommand{\arraystretch}{1.5}
\hspace{-0.3cm}
\begin{tabular}{p{3cm}p{10cm}}
$K$                   & number of arms \\
$n$                   & horizon \\
$t$                   & current time step \\
$\mu_i$               & expected return of arm $i$ \\
$\hat \mu_{i,s}$      & empirical estimate of return of arm $i$ based on $s$ samples \\
$\hat \mu_i(t)$       & empirical estimate of return of arm $i$ at time step $t$ \\
$\Delta_i$            & gap between the expected returns of the best arm and the $i$th arm \\ 
$\Delta_{\min}$       & minimum gap, $\Delta_2$ \\
$\logp(x)$            & $\max\set{1, \log(x)}$ \\
$H$                   & $\sum_{i=2}^K \Delta_i^{-2}$ \\
$H_i$                 & $\sum_{j=1}^K \min\set{\Delta_i^{-2}, \Delta_j^{-2}}$ \\ 
$c_\gamma$,$c_1,\ldots,c_{11}$     & non-negative constants (see \cref{app:constants}) \\
$\delta_T$, $\delta_\Delta$            & see \cref{def:delta} \\
$\tilde \delta_\Delta$        & see \cref{def:tildedelta} \\
$u_i$            & number of samples that we expect to choose suboptimal arm $i$ (see \cref{def:delta}) \\
$\gamma$        & ratio used in peeling argument \\
$\beta_{i,\Delta}$  & definition given in \cref{def:beta} \\
$\alpha, \psi$    & parameters used by \cref{alg:ocucb} \\
$F_\Delta$      & see \cref{def:failure} \\
$\tilde \Delta$ & see \cref{def:failure2}
\end{tabular}


\end{document}